\documentclass{article} 
\usepackage{geometry}
\date{}
\usepackage{caption}
\usepackage{graphicx}
\usepackage{caption}
\usepackage{tikz}
\usepackage{wrapfig}
\usepackage{amsmath,amsfonts,bm}

\def\1{\bm{1}}

\DeclareMathAlphabet{\mathsfit}{\encodingdefault}{\sfdefault}{m}{sl}
\SetMathAlphabet{\mathsfit}{bold}{\encodingdefault}{\sfdefault}{bx}{n}

\usepackage{cancel}

\usepackage{algorithmic}
\usepackage[final]{changes} %
\setdeletedmarkup{} %
\setaddedmarkup{\textcolor{Blue}{#1}}

\usepackage{hyperref}
\usepackage{url}

\usepackage{color, colortbl}
\definecolor{LightCyan}{rgb}{0.88,1,1}
\definecolor{Blue}{rgb}{0,0,1}
\definecolor{Gray}{gray}{0.9}
\usepackage{geometry}
\usepackage{multirow}
\usepackage{multicol}
\usepackage{microtype}
\usepackage{graphicx}
\usepackage{subfigure}
\usepackage{xcolor}
\usepackage{booktabs} 
\usepackage{amsmath, amsthm, amssymb, calrsfs, wasysym, verbatim, bbm, color, graphics, algorithm, algorithmic}
\usepackage{hyperref}
\usepackage[T1]{fontenc}
\usepackage[utf8]{inputenc}
\usepackage{bbm}

\newcommand{\specialcell}[2][c]{%
  \begin{tabular}[#1]{@{}c@{}}#2\end{tabular}}
\usepackage{xcolor} 
\usepackage{tikz}
\usetikzlibrary{fit,calc}

\colorlet{pink}{red!40}
\colorlet{blue}{cyan!60}
\newcommand\mydots{\hbox to 0.7em{.\hss.\hss.}}

\newtheorem{theorem}{{Theorem}}
\newtheorem{lemma}[theorem]{{Lemma}}

\DeclareMathAlphabet{\mathcal}{OMS}{cmsy}{m}{n}
\DeclareMathAlphabet{\mathbfsl}{OT1}{ppl}{b}{it} 

\newcommand{\bD}{\mathbfsl{D}}
 
\newcommand{\bF}{\mathbfsl{F}}

\newcommand{\bO}{\mathbfsl{O}}

\newcommand{\bd}{\mathbfsl{d}}

\newcommand{\bl}{\mathbfsl{l}}

\newcommand{\bv}{\mathbfsl{v}}
 
\newcommand{\bx}{\mathbfsl{x}}
 
\newcommand{\bz}{\mathbfsl{z}}

\newcommand{\bp}{\mathbfsl{p}}
\newcommand{\boldzero}{\textbf{0}}

\usepackage{color}

\title{\bf{Federated Learning with Heterogeneous Differential Privacy}}
\author{Nasser Aldaghri \\ \small University of Michigan\\ \footnotesize \texttt{\href{mailto:aldaghri@umich.edu}{aldaghri@umich.edu}} \and
Hessam Mahdavifar \\ \small University of Michigan\\ \footnotesize \texttt{\href{mailto:hessam@umich.edu}{hessam@umich.edu}} \and
Ahmad Beirami \\ \small MIT \\\footnotesize \texttt{\href{mailto:beirami@mit.edu}{beirami@mit.edu}}}

\begin{document}
\maketitle

\begin{abstract}
Federated learning (FL) takes a first step towards privacy-preserving machine learning by training models while keeping client data local. Models trained using FL may still indirectly leak private client information through model updates during training. Differential privacy (DP) may be employed on model updates to provide privacy guarantees within FL, typically at the cost of degraded performance of the final trained model. Both non-private FL and DP-FL can be solved using variants of the federated averaging (\textsc{FedAvg}) algorithm. In this work, we consider a heterogeneous DP setup where clients require varying degrees of privacy guarantees. First, we analyze the optimal solution to the federated linear regression problem with \emph{heterogeneous} DP in a Bayesian setup. We find that unlike the non-private setup, where the optimal solution for homogeneous data amounts to a single global solution for all clients learned through \textsc{FedAvg}, the optimal solution for each client in this setup would be a personalized one even for homogeneous data. We also analyze the privacy-utility trade-off for this problem, where we characterize the gain obtained from the heterogeneous privacy where some clients opt for less stringent privacy guarantees. We propose a new algorithm for FL with heterogeneous DP, referred to as \textsc{FedHDP}, which employs personalization and weighted averaging at the server using the privacy choices of clients, to achieve better performance on clients' local models. Through numerical experiments, we show that \textsc{FedHDP} provides up to $9.27\%$ performance gain compared to the baseline DP-FL for the considered datasets where $5\%$ of clients opt out of DP. Additionally, we show a gap in the average performance of local models between non-private and private clients of up to $3.49\%$, empirically illustrating that the baseline DP-FL might incur a large utility cost when not all clients require the stricter privacy guarantees.
\end{abstract}

\section{Introduction}
The abundance of data and advances in computation infrastructure have enabled the training of large machine learning (ML) models that efficiently solve a host of applications. It is expected that the majority of the generated data will be stored on devices with limited computational capabilities. To improve the privacy of user data and to reduce the amount of data transmission over networks in training ML models, \cite{mcmahan2017communication} proposed the federated learning (FL) paradigm to train a model using decentralized data at clients. 
Several prior works on FL algorithms have been proposed in the literature to overcome various issues that arise in realistic FL setups~\cite{kairouz2019advances, li2020federated, wang2021field}, especially with respect to data heterogeneity  \cite{konevcny2016federated,zhao2018federated,corinzia2019variational,hsu2019measuring,karimireddy2020scaffold,reddi2020adaptive}, and device dropout and communication cost \cite{li2018federated,zhu2019multi,wang2020federated,al2020federated}.

Despite the fact that client data is kept on device in FL, the trained model at the central server is still vulnerable to various privacy attacks, such as membership inference attacks \cite{shokri2017membership} and model inversion attacks \cite{fredrikson2015model}. Privacy-preserving variations of FL algorithms have been proposed in the literature, especially building on the concept of differential privacy (DP) to bound the amount of information leakage~\cite{dwork2014algorithmic}, while noting that using DP typically causes unavoidable degradation in performance. 
DP may be employed to provide privacy guarantees at different granularities: sample level (local DP), client level (global DP), or a set of clients (considered in this work), depending on the application at hand. %
Several prior works utilized DP to provide privacy guarantees for FL algorithms. For example, \cite{truex2020ldp, sun2020ldp, kim2021federated, song2015learning} apply DP mechanisms at clients to ensure local DP guarantees, where clients have complete control over the amount of privacy they desire. On the other hand, \cite{geyer2017differentially,mcmahan2018learning, andrew2019differentially, wei2020federated, bietti2022personalization} apply DP mechanisms at the server to ensure global DP guarantees for all clients. Applying DP typically causes some degradation in utility, i.e., the model's performance degrades as the privacy budget gets smaller \cite{alvim2011differential, sankar2013utility, makhdoumi2014information, calmon2015fundamental}. Other works include utilizing DP in cross-silo federated learning setups, such as \cite{lowy2021private, liu2022privacy}. These works commonly assume that the DP requirements of clients are homogeneous, which as we shall see may be overly strict and cause unnecessary degradation in performance when DP requirements are heterogeneous.

The concept of heterogeneous DP has been studied in the literature. \cite{avent2017blender, beimel2019power} consider a hybrid model by combining local DP with global DP and give clients the option to opt into either of these notions. A \emph{blender} process is considered by \cite{avent2017blender} for computing heavy hitters where some clients opt in local DP while the remaining opt in global DP. Some drawbacks of these works include their assumption of clients' data to be IID, as well as applying local DP which requires a large number of samples at clients. These assumptions make their approach inapplicable in FL setups due to the non-IID nature of clients' data in FL and the relatively small number of samples generated by clients in FL which requires either increasing the variance of the added noise or relaxing the privacy leakage budget leading to either large degradation in performance or higher privacy leakage budgets. Another related line of work, \cite{zhou2020bypassing, ferrando2021combining, liu2021projected, amid2022public}, studies combining public and private datasets to improve the utility of the model. However, it is not clear how to extend them to FL setups, where clients arbitrarily participate in training and only have access to their local datasets which are no longer IID. Finally, some theoretical results have also been derived about heterogeneous DP in the works \cite{alaggan2016heterogeneous, jorgensen2015conservative}, where mechanisms were created for such case. %

Heterogeneity is a central feature of FL. Data heterogeneity makes it hard to train a single global model that can effectively serve all clients \cite{li2020federated}, which can be remedied via model personalization~\cite{smith2017federated,wang2019federated,arivazhagan2019federated,khodak2019adaptive,mansour2020three,fallah2020personalized,deng2020adaptive,dinh2020personalized,li2021ditto}. Another type of heterogeneity include systems heterogeneity where different devices have different capabilities, in terms of various characteristics such as connection, computational, and power capabilities \cite{li2018federated}. Solutions to system heterogeneity include designing algorithms that can tolerate device dropout, reduce communication cost, or reduce computations cost \cite{caldas2018expanding, gu2021fast, horvath2021fjord, li2018federated}. In this work, we study heterogeneity along the privacy axis, which is relatively unexplored in FL.  In this paper, we theoretically study heterogeneous privacy requirements and show that model personalization could be used to learn good (even optimal) models in FL with heterogeneous DP.

\subsection*{Organization \& Our Contributions}
In this paper, we consider a new setup for privacy-preserving FL where privacy requirements are heterogeneous across  clients. We show that clients with less strict privacy requirements, even if they represent a small percentage of the overall population, can be leveraged to improve the performance of the global model learned via FL for all clients. Our contributions and the organization of the paper are as follows:

\vspace{-0.12in}
\begin{itemize}
\itemsep0em
    \item In Section \ref{FedHDP_general}, we propose a heterogeneous setup for privacy in FL. Instead of granting the same level of privacy for all clients, each client is given the option to choose their level of privacy. Moreover, we formally pose an optimization objective for solving the problem from every client's point of view.
    
    \item In Section \ref{FLR}, we theoretically study heterogeneous DP in the simplified Bayesian setup of federated linear regression, introduced by \cite{li2021ditto}, where clients are either private or non-private (i.e., two levels of privacy). Unlike the case of \emph{non-private FL with homogeneous data\footnote{Homogeneous data refers to the case where the data for all clients is independent and identically distributed (IID).}}, where the Bayes optimal solution is a single global model that could be learned via vanilla federated averaging, the optimal Bayes solution in differentially-private FL requires personalization, even in the case of homogeneous DP. %
    Further, we characterize the privacy-utility trade-off observed at clients.

    \item In Section \ref{FedHDP}, we formally propose the FL with heterogeneous DP algorithm, referred to as \textsc{FedHDP}, for the heterogeneous privacy setup. The \textsc{FedHDP} algorithm extends the Bayes optimal solution for federated linear regression to be applicable to more realistic scenarios.
 
    \item In Section \ref{experiments}, we provide experimental results of the \textsc{FedHDP} algorithm using various synthetic and realistic federated datasets from TensorFlow Federated (TFF) \cite{TFF} using reasonable privacy parameters, where the privacy choices presented to clients are either to be private or non-private. Although the design guarantees of \textsc{FedHDP} don't apply in these complex settings, we experimentally show that it provides significant gains compared to \textsc{DP-FedAvg} algorithm \cite{andrew2019differentially}, and other stronger variants of it.
\end{itemize}

\section{Privacy Guarantees within Federated Learning}\label{FedHDP_general}
We briefly describe the federated learning (FL) setup together with existing privacy-preserving algorithms and their privacy guarantees.

\textbf{Federated learning:} FL consists of a central server that wishes to learn a global model for a set of clients denoted by $\mathcal{C}$. The clients cooperate with the server to learn a model over multiple training rounds while keeping their data on device. Each client $c_j \in \mathcal{C}$ has a local loss $f_j(\cdot)$ and a local dataset $\bD_j=\{\bd_{j_1},\bd_{j_2},\mydots,\bd_{j_{n_j}}\}$, where $\bd_{j_i}$ is the $i$-th sample at the $j$-th client. During communication round $t$, the server sends the current model state, i.e., $\boldsymbol{\theta}^t$, to the set of available clients in that round, denoted by $\mathcal{C}^t$, who take multiple gradient steps on the model using their own local datasets to minimize their local loss function $f_j(\cdot)$. The clients then return the updated model to the server who aggregates them, e.g., by taking the average, to produce the next model state $\boldsymbol{\theta}^{t+1}$. This general procedure describes a large class of learning global models with FL, such as federated averaging (\textsc{FedAvg}) \cite{mcmahan2017communication}.

\textbf{Privacy-preserving FL:} To design privacy-preserving FL algorithms using DP, certain modifications to the baseline \textsc{FedAvg} algorithm are required. In particular, the following modifications are introduced: clipping and noising. Considering client-level privacy, the averaging operation at the server is the target of such modifications. Suppose that clients are selected at each round from the population of all clients of size $N$, with a certain probability denoted by $q$. First, each client update is clipped to have a norm at most $S$, then the average is computed followed by adding a Gaussian noise with mean zero and co-variance $\sigma^2 I = z^2(\frac{S}{qN})^2 I$. The variable $z$ is referred to as the noise multiplier, which dictates the achievable values of $(\epsilon, \delta)$-DP. Training the model through multiple rounds increases the amount of leaked information. Luckily, the moment accountant method in \cite{abadi2016deep} can be used to provide a tighter estimate of the resulting DP parameters $(\epsilon, \delta)$. This method achieves client-level DP.  %
It is worth noting that the noise can be added on the client side but needs to achieve the desired resulting noise variance in the output of the aggregator at the server, which is still the desired client-level DP. Selecting the clipping threshold as well as the noise multiplier is essential to obtaining useful models with meaningful privacy guarantees. 

One issue that comes up during the training of DP models is the norm of updates can either increase or decrease; if the norm increases or decreases significantly compared to the clipping norm, the algorithm may slow down or diverge. \cite{andrew2019differentially} presented a solution to privately and adaptively update the clipping norm during each round of communication in FL based on the feedback from clients on whether or not their update norm exceeded the clipping norm. We consider this as the baseline for the privacy-preserving FL algorithm and refer to it in the rest of the paper as \textsc{DP-FedAvg}~\cite{andrew2019differentially}. The case where no noise is added is the baseline for non-private FL algorithm, which is referred to simply as \textsc{Non-Private}.

One fundamental aspect of \textsc{DP-FedAvg} is that it provides an equal level of privacy to \emph{all} clients. This is suitable for the case when all clients have similar behavior towards their own privacy in the FL setup. In other words, \textsc{DP-FedAvg} implicitly assumes a homogeneity of the privacy level is required by all clients. This is in contrast to the heterogeneity feature of FL setups, where different clients have different data, capabilities, and objectives, which also applies to privacy choices. Next, we describe our proposed setup for FL with heterogeneous DP.

\subsection*{Proposed Setup: Heterogeneous Privacy within Federated Learning}
Let us describe the proposed setup for FL with heterogeneous DP. We first start by discussing the privacy guarantees in the setup, then discuss the objectives in such setup from the server's and clients' points of view.

\textbf{Privacy guarantees:} Prior to training, each client $c_i \in \mathcal{C}$ chooses their differential privacy parameters $(\epsilon_i,\delta_i)$. We investigate what the server and clients agree upon at the beginning of training an FL model in terms of privacy to formally define the considered privacy measures. Each client $c_j$, whose dataset is denoted as $\bD_j$ that is disjoint from all other clients, requires the server to apply some randomized algorithm $A_j(\cdot)$, whose image is denoted as $\bO_j$, such that the following holds
\begin{align}
    \Pr(A_j(\bD_j) \in O_j) \leq e^{\epsilon_j} \Pr(A_j(\bD_e) \in O_j) + \delta_j, \label{client_privacy_condition} 
\end{align}
where $\bD_e$ is the empty dataset, and the relationship holds for all subsets $O_j \subseteq \bO_j$. This achieves client-level privacy with parameters $(\epsilon_j,\delta_j)$ from client $c_j$'s point of view. Now, assume we have $N$ clients each with their own privacy requirements for the server $(\epsilon_j,\delta_j)$ for $j \in [N]$, which should hold regardless of the choices made by any other client. Let us have a randomized algorithm $A(\cdot)$, which denotes the composition of all $A_j(\cdot)$'s; then, the parallel composition property of DP states that the algorithm $A(\cdot)$ is $(\epsilon_c,\delta_c)$-DP, which satisfies the following:
\begin{align}
    \Pr(A(\bD) \in O) \leq e^{\epsilon_c} \Pr(A(\bD') \in O) + \delta_c, \label{composed_privacy_condition}
\end{align}
where $\bD$ contains all datasets from all clients and $\bD'$ contains datasets from all clients but one, $\bO$ is the image of $A(\cdot)$, and the relationship holds for all neighboring datasets $\bD$ and $\bD'$ that differ by only one client and all $O \subseteq \bO$. The parallel composition property of DP states that the resulting  $\epsilon_c = \max_i \epsilon_i$, and $\delta_c = \max_i \delta_i$. Next, considering our setup, let us have $l$ sets of private clients $\mathcal{C}_i$'s. Each client in the $i$-th set of clients requires $(\epsilon_{p_i},\delta_{p_i})$-DP, and without loss of generality, assume that $\epsilon_{p_i} \geq \epsilon_{p_l}$ and $\delta_{p_i} \geq \delta_{p_l}$ $\forall i<l$. This is the case we consider in this paper, where we apply a randomized algorithm $A_{p_i}(\cdot)$, whose image is denoted as $\bO_{p_i}$, to the dataset that includes all clients in the set $\mathcal{C}_i$ and the following holds
\begin{align}
    \Pr(A_{p_i}(\bD_{p_i}) \in O_{p_i}) \leq e^{\epsilon_{p_i}} \Pr(A_{p_i}(\bD_{p_i}') \in O_{p_i}) + \delta_{p_i}, \label{FedHDP_privacy_condition}
\end{align}
where $\bD_{p_i}$ contains all datasets from all clients in $\mathcal{C}_i$ and $\bD_{p_i}'$ contains datasets from all clients in that subset except one, and the relationship holds for all neighboring datasets $\bD_{p_i}$ and $\bD_{p_i}'$ that differ by only one client and all $O_{p_i} \subseteq \bO_{p_i}$.

Now, let us assume in the proposed heterogeneous DP setup that each client in $\mathcal{C}_i$ requires $(\epsilon_{p_i},\delta_{p_i})$-DP in the sense of \eqref{client_privacy_condition}. As a result, we can see that the only way for \textsc{DP-FedAvg} to \emph{guarantee} meeting the privacy requirement for the clients in $\mathcal{C}_l$ with $(\epsilon_{p_l},\delta_{p_l})$ is to enforce $(\epsilon_{p_l},\delta_{p_l})$-DP for \emph{all} clients. In other words, \textsc{DP-FedAvg} needs to be $(\epsilon_{p_l},\delta_{p_l})$-DP, i.e., it needs to apply the strictest privacy parameters to all clients in the sense of \eqref{composed_privacy_condition}. On the other hand, in our setup, we can \emph{guarantee} meeting the privacy requirements for each set of clients by ensuring an $(\epsilon_{p_i},\delta_{p_i})$-DP for clients in $\mathcal{C}_i$, respectively, in the sense of \eqref{FedHDP_privacy_condition}. In other words, we need to only apply the appropriate DP algorithm with its appropriate metrics for each subset of clients to ensure the privacy metrics are met. This in turn results in our setup satisfying the corresponding privacy requirements needed by each set of clients, which are the main targets that need to be achieved in both algorithms from the clients' point of view in terms of their desired privacy levels.

\textbf{Objectives in FL setup:} In terms of objectives in FL setups, the server's goal is to utilize clients' updates, which may be subject to specific DP conditions, by averaging them to produce the next model state. On the other hand, clients have a different objective when it comes to their performance measures. The clients' goal is to minimize their loss function given all other clients' datasets including their own. However, since clients do not have access to other clients' raw data, a client desires to use the information from the differentially-private updates computed using the datasets by other clients as well as its own local update in order to reach a solution. Assume that the client $c_j$ observes all other clients' DP-statistics of the datasets $\{\boldsymbol{\psi}_i: i \in [N] \backslash j\}$, which are the outputs of a randomized function that satisfies the privacy condition in \eqref{client_privacy_condition}, as well as its own non-DP dataset $\bD_j$. Then the client's Bayes optimal solution is
\begin{align}
    \hspace{0.95in} \boldsymbol{\theta}_j^* = \arg \min_{\boldsymbol{\widehat{\theta}}_j} \left\{ \mathbb{E}_{\mathcal{D}_j} \left[ \ell_j(\boldsymbol{\widehat{\theta}}_j) \big| \{ \boldsymbol{\psi}_i: i\in [N] \backslash j\},  \bD_j \right]\right\}~~~~~~~~~~~~~~~~~~~ \tag{Local Bayes obj.}
    \label{eq:local-bayes}
\end{align}
where $\ell_j (\cdot)$ is the loss function used to train a model for client $c_j$, and $\mathcal{D}_j$ is the true distribution of the dataset at client $c_j$. Notice that client $c_j$ has access to their own dataset $\bD_j$ and DP-statistics of the other datasets $\{\widetilde{\bD}_i: i \in [N] \backslash j\}$. From client $c_j$'s point of view, this objective is the Bayes optimal solution when observing all DP-statistics from other clients that are subject to their privacy constraints. It is important to note that~\eqref{eq:local-bayes} is not suited for a federated optimization framework due to the fact that even individual updates from other clients are not available at the client to utilize, but rather their aggregation through a global model at the server. In practice, each client utilizes the global model $\widehat{\boldsymbol{\theta}}$ to optimize the following:
\begin{align}
   ~~~ \widehat{\boldsymbol{\theta}}_j^* = \arg \min_{\boldsymbol{\widehat{\theta}}_j} \left\{ \mathbb{E}_{\mathcal{\bD}_j} \left[ \ell_j(\boldsymbol{\widehat{\theta}}_j) \big| \widehat{\boldsymbol{\theta}}, \bD_j \right]\right\}. \label{local_bayes_objective_general} \tag{Local personalized federated obj.}
\end{align}
We notice that this solution is a form of personalization in FL, where clients no longer deploy the global model locally by default, but rather utilize it to derive local, preferably better, models that perform well on their own local dataset. In the remainder of this paper, we will demonstrate this approach's ability to learn good (even optimal as we shall see in the next section) personalized local models compared to baseline private FL. Next, we will consider the proposed setup for a simplified federated problem known as the federated linear regression.

\section{Analyzing Heterogeneous Differential Privacy in Simplified Settings}\label{FLR}
In this section, we study the heterogeneous DP problem in a simplified linear regression setup inspired by the one proposed by \cite{li2021ditto}. 
We consider a federated linear regression setup, where clients are interested in learning Bayes optimal models in the sense of \eqref{eq:local-bayes}. We briefly show that in this simplified setup, the solution could be cast as a bi-level optimization problem, which can be solved as a personalized FL problem \eqref{local_bayes_objective_general}, which motivates our subsequent algorithm design. 

We first start by considering the global estimation on the server and show the proposed solution is Bayes optimal. Then, we consider local estimations for all clients and show that the proposed solution is Bayes optimal for all clients when using appropriate values of the respective hyperparameters. We further characterize the privacy-utility trade-off gains in the Bayes optimal solution compared to variants of DP federated averaging \cite{andrew2019differentially}. The discussions in this section are brief, refer to the appendix for further and more comprehensive discussions.

\subsection{Analyzing Heterogeneous Differential Privacy in Federated Linear Regression}\label{sec:FLR}
We consider a learning setup consisting of $N$ clients, each with a fixed number of data points, where the goal of each client is to estimate $d$-dimensional linear regression coefficient vector. In linear regression, clients have observations of feature vectors, each of length $d$, and a response variable. The features linearly combine to produce the response variable, where the goal is to find the coefficients vector of length $d$. To encapsulate the property of data heterogeneity in federated learning setups, assume that such coefficient vectors are correlated across clients but they are not the same. Specifically, assume a fixed coefficient vector to be estimated at the server $\boldsymbol{\phi}$ drawn from some distribution, and assume clients' coefficient vectors $\boldsymbol{\phi}_j$'s, which will be estimated at clients, are correlated observations of such vector. This can be modeled by assuming clients' observations are noisy versions of the coefficient vector at the server, where the noise is drawn from a Gaussian distribution and its covariance controls the level of heterogeneity of data at clients. Additionally, to ensure client-level privacy, assume the algorithm is required to provide a specific client-level privacy for each subset of clients by adding Gaussian noise to the clients' updates with an appropriate covariance. Without loss of generality, assume the privacy noise is added on the client's side which achieves the desired level at the server after the aggregation of updates. Denote the estimate of $\boldsymbol{\phi}_j$ at client $c_j$ as $\widehat{\boldsymbol{\phi}}_j$, and its DP update sent to the server as $\boldsymbol{\psi}_j$, then the server's and clients' goals are to minimize the Bayes risk (i.e., test error), defined as follows
\begin{gather}
    \boldsymbol{\theta}^* := \arg \min_{\widehat{\boldsymbol{\theta}}} \left\{ \mathbb{E} \left[\left. \frac{1}{2} \|\widehat{\boldsymbol{\theta}} -\boldsymbol{\phi}\|_2^2  \right| \boldsymbol{\psi}_1, \mydots, \boldsymbol{\psi}_N \right] \right\}, \label{global_regression1}\\
    \hspace{1.2in} \boldsymbol{\theta}^*_j := \arg \min_{\widehat{\boldsymbol{\theta}}} \left\{ \mathbb{E} \left[\left. \frac{1}{2} \| \widehat{\boldsymbol{\theta}}-\boldsymbol{\phi}_j \|_2^2  \right| \{\boldsymbol{\psi}_{i} : i \in  [N]  \setminus j \}, \hat{\boldsymbol{\phi}}_j \right] \right\},\tag{Local Bayes obj.}\label{local_regression2}
\end{gather}
where \eqref{global_regression1} corresponds to the server's Bayes objective, given all clients' updates, while \eqref{local_regression2} is the client's Bayes objective, given its non-private estimation as well as other clients' private updates. From now on, we refer to our solution and algorithm as \textsc{FedHDP} in the remainder of the paper. The algorithm's pseudocode for federated linear regression is described in Algorithm \ref{FedHDP_FLR_algorithm}. For the federated linear regression setup, the server's goal is to find the following:
\begin{align}
    \widehat{\boldsymbol{\theta}}^* := \arg \min_{\widehat{\boldsymbol{\theta}}} \left\{   \frac{1}{2} \left\|\sum_{i \in [N]} w_i \boldsymbol{\psi}_i-\widehat{\boldsymbol{\theta}} \right\|_2^2 \right\}. \label{global_objective_opt2}
\end{align}
However, as for the clients in the considered federated setup, they don't have access to individual updates from other clients, but rather have the global estimate $\widehat{\boldsymbol{\theta}}^*$. So, we have the local \textsc{FedHDP} objective as
\begin{align}
    \hspace{1.2in} \widehat{\boldsymbol{\theta}}^*_j := \arg \min_{\widehat{\boldsymbol{\theta}}} \left\{ \frac{1}{2} \|\widehat{\boldsymbol{\theta}} - \widehat{\boldsymbol{\phi}}_j \|_2^2  +  \frac{\lambda}{2} \|\widehat{\boldsymbol{\theta}}-\widehat{\boldsymbol{\theta}}^*\|_2^2 \right\}, \tag{Local FedHDP obj.}\label{local_objective_opt_2}
\end{align}
where $\lambda_j$ is a regularization parameter that controls the closeness of the personalized model towards the global model. Higher values of $\lambda_j$ steer the personalized model to the global model, while smaller values of $\lambda_j$ steer the personalized model towards the local model at the client. Notice that~\eqref{local_objective_opt_2} is a special case of the \eqref{local_bayes_objective_general} where personalization is performed through a bi-level regularization. Refer to Appendix \ref{LR_setup} for the full description of all variables and parameters considered in the federated linear regression setup. Next, we provide a brief discussion on the optimality of the \textsc{FedHDP} and state its convergence in this setup to the Bayes optimal solution for the server as well as the clients.

\begin{algorithm*}[th!]
    \caption{\textsc{FedHDP:} Federated Learning with Heterogeneous Differential Privacy (Linear Regression)}\vspace{-0.2in}
    \begin{multicols}{2}
    \begin{algorithmic}\label{FedHDP_FLR_algorithm}
        \STATE \textit{Inputs:} $\boldsymbol{\theta}^0$, $\sigma_c^2$, $\gamma^2$, $\eta=1$, $\{\lambda_j\}_{j \in [N]}$, $r$, $\rho_{\text{np}}$, $N$.
        \STATE \textit{Outputs:} $\boldsymbol{\theta}^*, \{\boldsymbol{\theta}_j^*\}_{j \in [N]}$
        \STATE \textbf{At server:}
        \FOR{client $c_j$ in $\mathcal{C}^t$ \textbf{in parallel}}
        \STATE $\boldsymbol{\psi}_j \leftarrow$ \textit{ClientUpdate}($\theta^t, c_j$)
        \ENDFOR
        \STATE $\boldsymbol{\theta}^* \leftarrow \frac{1}{\rho_{\text{np}} N+r(1-\rho_{\text{np}}) N} \sum_{c_i \in \mathcal{C}_{\text{np}}} \boldsymbol{\psi}_i$
        \STATE \hspace{0.33in} $+\frac{r}{\rho_{\text{np}} N+r(1-\rho_{\text{np}}) N} \sum_{c_i \in \mathcal{C}_{\text{p}}} \boldsymbol{\psi}_i$
        \columnbreak
        \STATE \textbf{At client $c_j$:}
        \STATE \textit{ClientUpdate}($\boldsymbol{\theta}^0, c_j$):
        \STATE $\boldsymbol{\theta} \leftarrow \boldsymbol{\theta}^0$
        \STATE $\boldsymbol{\theta}_j \leftarrow \boldsymbol{\theta}^0$
        \STATE $\boldsymbol{\theta} \leftarrow \boldsymbol{\theta} - \eta (\boldsymbol{\theta}-\frac{1}{n_s}\sum_{i} \bx_{j,i} )$
        \STATE $\boldsymbol{\theta}_j^* \leftarrow \boldsymbol{\theta}_j - \eta_j \big( (\boldsymbol{\theta}_j-\frac{1}{n_s}\sum_{i} \bx_{j,i})+\lambda_j (\boldsymbol{\theta}_j-\boldsymbol{\theta}) \big)$
        \STATE $\boldsymbol{\psi} \leftarrow \boldsymbol{\theta} + \mathbbm{1}_{c_j \in \mathcal{C}_{\text{p}}} \mathcal{N}(0,(1-\rho_{\text{np}}) N\gamma^2)$
        \STATE return $\boldsymbol{\psi}$ to server
    \end{algorithmic}
    \end{multicols}
\end{algorithm*}

\subsubsection{Federated Linear Regression with Private and Non-Private Clients}\label{FLR_opt_out}
In our discussion so far, we assumed that clients have multiple privacy levels to choose from. In realistic setups, clients are expected to be individuals who may not have complete awareness of what each parameter means in terms of their privacy. In fact, interpreting the meaning of $\epsilon$ in DP is hard in general \cite{dwork2019differential}, and thus it is unrealistic to assume the average client can make their choice precisely. Therefore, the server needs to make a choice on how these parameters are presented to clients. A special case we consider extensively in this paper is the case with two classes of \emph{private} and \emph{non-private} clients, with $\rho_{\text{np}}$ fraction of the clients being non-private. Clients who choose to be private, denoted by the subset $\mathcal{C}_{\text{p}}$, are guaranteed a fixed $(\epsilon, \delta)$-DP, while clients who choose otherwise, denoted by the subset $\mathcal{C}_{\text{np}}$, are not private. In fact, a practical approach to privacy from the server's point of view is to enable privacy by default for all clients and give each client the option to opt out of privacy if they desire to do so. The non-private choice can be suitable for different types of clients such as enthusiasts, beta testers, volunteers, and company employees, among others. Before we present our proposed solution, we will discuss two major baselines that we consider in the remainder of the paper:

\textbf{\textsc{DP-FedAvg}} uses the differentially private federated averaging \cite{andrew2019differentially} as the global model, where all clients are guaranteed the same level of privacy regardless of their preference. %

\textbf{\textsc{HDP-FedAvg}} ensures privacy for the set of private clients similar to \textsc{DP-FedAvg}, however, no noise is added to the non-private clients' updates. %
By design, \textsc{HDP-FedAvg} is a stronger variant of \textsc{DP-FedAvg} for a more fair comparison in a heterogeneous privacy setting. %

\textbf{\textsc{FedHDP}} is the algorithm we propose in this paper. There are three important parameters introduced in the algorithm: $r$ (the ratio of the weight assigned to private clients to that of the non-private ones server-side), $\gamma^2$ (the variance of the noise added for privacy), and $\lambda_j$ described earlier. Furthermore, utilizing the assumption of a diagonal covariance matrix, which was also made by \cite{li2021ditto}, denote the noisy observations of the vector $\boldsymbol{\phi}_j$ at client $c_j$ by $\bx_{j,i}$ for $i=1,\mydots,n_s$. In the remainder of this section, we briefly analyze Algorithm \ref{FedHDP_FLR_algorithm}, and discuss its Bayes optimality for the federated linear regression problem for a specific privacy parameter $\gamma^2$.

\subsubsection{Optimal Global Estimate on the Server}\label{FLR_global}
The server's goal is to combine the updates received from clients such that the resulting noise variance is minimized, while ensuring the privacy of the set of private clients. To this end, we build on and extend the approach presented by \cite{li2021ditto} to find the optimal aggregator at the server. The server first computes the two intermediate average values for non-private and private clients such that they achieve the desired levels of privacy for each group. Now, the server aims to combine such values to compute its estimation $\theta$ of the value of $\phi$ with the goal of minimizing the resulting estimation noise covariance. Lemma \ref{lemma2_ditto} in Appendix \ref{LR_global} states the optimal aggregator, in this case, is the weighted average of such updates. If the resulting weighted average is expanded, it can be expressed as $\sum_{i = [N]} w_i \psi_i$. In this case, considering the weights used in the weighted average, let us denote the ratio of weights $w_i$'s dedicated for private clients to weights for non-private clients by $r=\frac{w_\text{p}}{w_{\text{np}}}$. The aim here is to put more weight for the non-private clients compared to the private clients, so that the information provided by them is utilized more heavily in the global estimate. Lemma \ref{opt_global_estimate} in Appendix \ref{LR_global} states the Bayes optimality of the proposed solution and derives the optimal value of $r^*$ for the considered setup when substituting the corresponding values of such parameters.

\subsubsection{Optimal Local Estimates on Clients}\label{FLR_personalized}
Now that we have stated the optimal server-side aggregation, we consider optimal learning strategies on the client side. In particular, we find that \textsc{FedHDP} achieves the Bayes optimal solution \eqref{eq:local-bayes} for local estimates at both the private as well as the non-private clients when using the optimal value of the ratio hyperparameter $r^*$ as well as the regularization parameters $\lambda_j$'s for each client. Lemma \ref{opt_personalized_estimate} in Appendix \ref{LR_local} states the Bayes optimality of this solution and states the optimal values of each of the regularization parameters $\lambda_j^*$. It is worth mentioning that this result shows that \textsc{FedHDP} recovers the~\eqref{eq:local-bayes}, which is the best one could hope for even if the client had access to DP versions of all other clients' updates without any constraints that arise in FL, such as just having access to the global model. We also notice that the values of $\lambda^*$ are different for private and non-private clients. We recall that during the computation of the expressions for the personalization parameters for all clients we considered the presence of data heterogeneity as well as privacy heterogeneity. In Table \ref{special_cases_lambda_LR} we provide a few important special cases for the \textsc{FedHDP} algorithm for the considered federated linear regression problem.

\begin{table}[ht!]
\captionsetup{font=small}
	\caption{Special cases of \textsc{FedHDP} in the federated linear regression.}
	\vspace{-.15in}
	\label{special_cases_lambda_LR}
	\begin{center}
		\footnotesize
		\begin{tabular}{|l|ll|ll|}
            \hline
            \rowcolor{Gray} & \multicolumn{2}{c|}{\textbf{No privacy %
            }} & \multicolumn{2}{c|}{\textbf{Homogeneous privacy %
            }} \\ \hline
			\cellcolor{Gray} \textbf{Homogeneous data %
            }   & \textsc{FedAvg}~~~ &  %
            & \textsc{DP-FedAvg+Ditto}~~~ & %
            \\ \hline
			\cellcolor{Gray} \textbf{Heterogeneous data %
            } &  \textsc{FedAvg+Ditto}~~~ & %
            & \textsc{DP-FedAvg+Ditto}~~~ & %
            \\ \hline
		\end{tabular}
	\end{center}
 \vspace{-.1in}
\end{table}

Finally, we state the optimality of the \textsc{FedHDP} solution for both estimates, i.e., local and global, using the appropriate learning rate in Theorem \ref{FedHDP_optimality} in Appendix \ref{convergence_LR}.

\textbf{Remark:} Although the heterogeneous DP problem is fundamentally different from robustness to data-poisoning attacks \cite{li2021ditto}, its solution bears resemblance to a recently-proposed personalization scheme known as \textsc{Ditto} \cite{li2021ditto}. \textsc{FedHDP}, in its general form, differs from \textsc{Ditto} in a number of major ways, and recovers it as a special case. First, the server-side aggregation in \textsc{Ditto} is the vanilla \textsc{FedAvg}; however, in the proposed solution the server-side aggregation is no longer \textsc{FedAvg}, but rather a new aggregation rule which utilizes the privacy choices made by clients. Second, \textsc{Ditto}, where the setup includes two sets of clients, i.e., benign and malicious, is designed for robustness against malicious clients; hence, the performance on malicious clients is not considered. On the other hand, the sets of clients in our proposed setup are the sets of clients with different levels of privacy; hence, measuring the performance across all sets of clients, i.e., clients with different privacy levels, is needed, and improving their performance is desired across all sets of clients. Third, the server in \textsc{Ditto} is unaware of the status of the clients, i.e., whether or not they are malicious; while in the proposed setup the server is aware of the privacy choices made by clients, and hence can give different weights to updates from private and non-private clients.

\subsubsection{Privacy-Utility Trade-off}\label{FLR_tradeoff}
Thus far, we have shown that for the problem of federated linear regression, the global estimate can benefit from the introduced setup of heterogeneous DP. A better global estimate would enable better performance on clients' devices in the federated linear regression setup, even when no personalization is utilized. However, a question may arise on whether clients have a utility cost if they choose to remain private compared to the case where they opt out and become non-private.

\begin{figure}
	\captionsetup{font=small}
	\centering
	\includegraphics[width=0.45\textwidth]{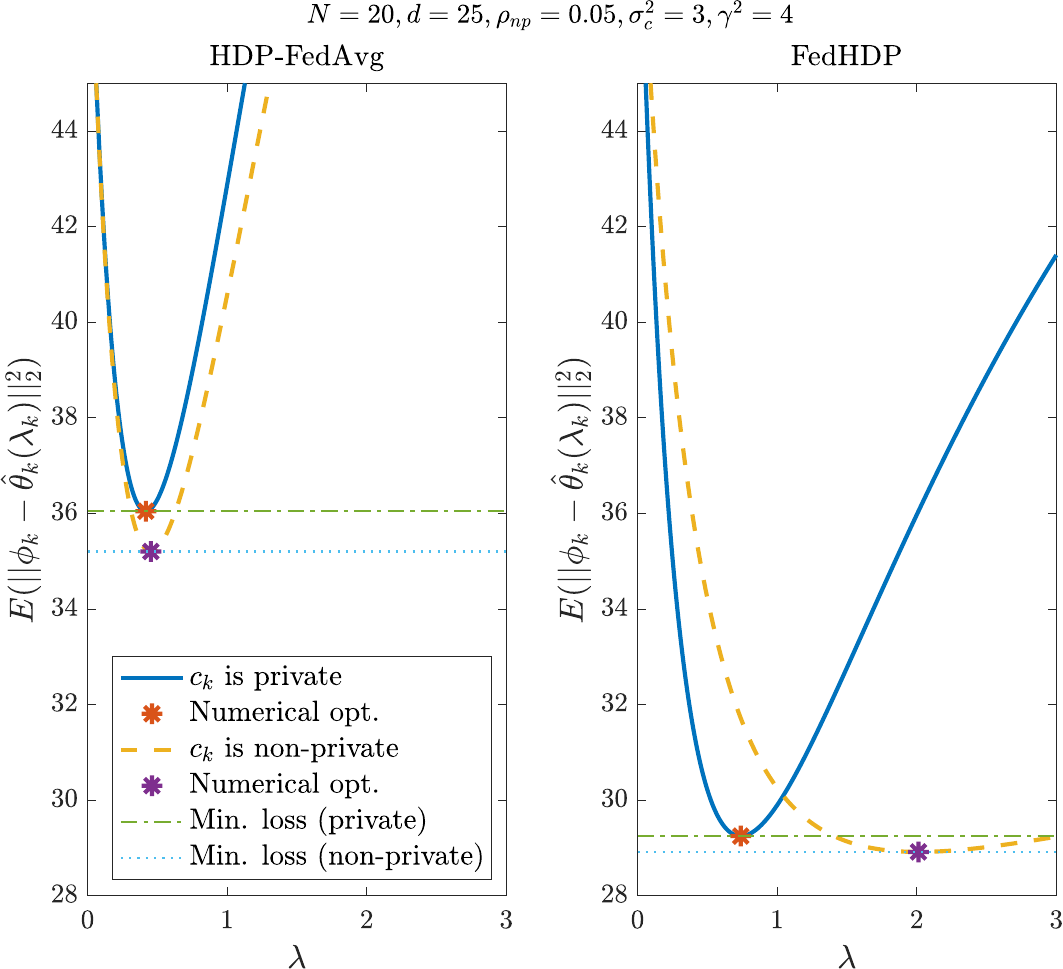} 
	\caption{The effect of opting out on the personalized local model estimate for a linear regression problem as a function of $\lambda$ when employing (left) \textsc{HDP-FedAvg+Ditto} and (right) \textsc{FedHDP}. Note that $\sigma_c^2$ includes the non-IID parameter, and $\gamma^2$ is the variance of the privacy noise.}
	\label{gain_opting_out_LR2}
\end{figure}

To answer this question, we argue that opting out helps the server produce a better global estimate, in addition to helping clients produce better personalized local estimates. In other words, clients that opt out can produce better personalized local estimates compared to the clients that remain private.

To illustrate the \emph{motivation} of opting out for clients, we perform an experiment where we conduct the federated linear regression experiment for two scenarios. The first is the case where client $c_k$ remains private, and the second is the case where $c_k$ opts out of privacy and becomes non-private. For comparison, we provide the results of the experiments of \textsc{FedHDP} with the optimal value $r^*$, as well as \textsc{HDP-FedAvg+Ditto}. The results of these experiments are shown in Figure \ref{gain_opting_out_LR2}.%
We can see that if the client is non-private, they exhibit improvements in their estimates using the optimal value $\lambda^*$ for both algorithms, but the proposed \textsc{FedHDP} with the optimal value $r^*$ greatly outperforms \textsc{HDP-FedAvg+Ditto}. Additionally, in this problem, we can see that the optimal value of $\lambda_{\text{np}}^*$ for non-private clients is always greater than or equal to the value $\lambda_{\text{p}}^*$ for private clients, which is due to the value of $r$ being less than or equal to $1$. In other words, non-private clients have more influence on the global estimate, and hence, encouraging the local estimate to get closer to the global estimate is more meaningful compared to private clients. Furthermore, this experiment illustrates an important trade-off between privacy and utility for each client, where opting out of privacy improves performance, while maintaining privacy incurs degraded performance.

\subsection{Extension to Federated Linear Regression with Multiple Privacy Levels}
We presented brief results of the analysis for the federated linear regression with two privacy levels, i.e., private and non-private, as a first step towards demonstrating the effectiveness of \textsc{FedHDP}. We present the complete analysis of the federated linear regression and extend such analysis to the following setups:
\begin{itemize}
\itemsep0em
    \item {\bf Federated linear regression with two privacy levels:} In this extended setup, we have two subsets of clients $\mathcal{C}_1$ and $\mathcal{C}_2$, each having its own privacy requirements $\gamma^2_1$ and $\gamma^2_2$, respectively. We perform an analysis of the new setup, and derive the expressions for the optimal hyperparameters under a diagonal covariance matrix assumption. You can refer to Appendix \ref{LR_appendix} for more. Additionally, in Appendix \ref{fpe_appendix}, we provide a summarized analysis of the federated point estimation where we highlight some insightful observations.
    \item {\bf Federated linear regression beyond two privacy levels:} In Appendix \ref{app:more-than-two}, we consider the case with more than two privacy levels in federated linear regression, provide an analysis of this setup, and show the solution to the optimal aggregator as well as the optimal regularization parameters for each set of clients. The solution would still be achieved by \textsc{FedHDP} algorithm, however, now with different values of $\lambda_j$ for each set of clients, and a server-side weighted averaging that depends on the individual values of noise variance $\gamma^2_i$ of each set of clients.
\end{itemize}

\section{\textsc{FedHDP}: Federated Learning with Heterogeneous Differential Privacy }\label{FedHDP}
Now that we have been able to find a Bayes optimal solution in the simplified federated linear regression setup (and some extensions of it), we build upon the ingredients we used to build a general solution for FL with heterogeneous DP. 
We formally present the \textsc{FedHDP} algorithm and elaborate on its hyperparameters. The \textsc{FedHDP} algorithm that is designed to take advantage of the aforementioned heterogeneous privacy setup is described in Algorithm \ref{FedHDP_algorithm}. Similarly to the simplified setting, \textsc{FedHDP} utilizes DP with adaptive clipping, upweighing of less private clients on the server side, and a simple form of personalization.

\begin{algorithm*}[tbh]
\centering
\small
    \caption{\textsc{FedHDP}: Federated learning with heterogeneous DP (General Algorithm)}\vspace{-0.2in}
    \begin{multicols}{2}
    \begin{algorithmic}\label{FedHDP_algorithm}
        \STATE \textit{Inputs:} model parameters $\boldsymbol{\theta}^0$, sensitivity $S^0$, learning rate $\eta$, personalized learning rate $\eta_p$, noise multipliers $\boldsymbol{z}, z_b$, quantile $\kappa$, and factor $\eta_b$.
        \STATE \textit{Outputs:} $\boldsymbol{\theta}^{T}, \{\boldsymbol{\theta_j}\}_{j \in [N]}$
        \STATE \textbf{At server:}
        \FOR{$t=0,1,2,...,T-1$}
        \STATE $\mathcal{C}^t \leftarrow$ Sample $N^t$ clients from $\mathcal{C}$
        \FOR{client $c_j$ in $\mathcal{C}^t$ \textbf{in parallel}}
        \STATE $\bigtriangleup \boldsymbol{\theta}^{t}_j, b^t_j \leftarrow$ \textit{ClientUpdate}($\boldsymbol{\theta}^t, c_j, S^t$)
        \ENDFOR
        \FOR{$j \in [l]$ \textbf{in parallel}}
        \STATE $N_{j}^t \leftarrow |\mathcal{C}_{j}^t|$, \hspace{0.05in} $z_j^t \leftarrow z_j \frac{S^t}{N_{j}^t}$
        \STATE $\bigtriangleup\tilde{\boldsymbol{\theta}}^{t}_j \leftarrow \frac{1}{N_{j}^t} \sum_{c_i \in \mathcal{C}_{j}^t} \bigtriangleup \boldsymbol{\theta}^{t}_i+ \mathcal{N}(\textbf{0},(z_j^t)^2 \textbf{I})$
        \ENDFOR
        \STATE $\bigtriangleup\boldsymbol{\theta}^{t} \leftarrow \sum_{i \in [l]} w_i^t \bigtriangleup\tilde{\boldsymbol{\theta}}^{t}_i$
        \STATE $\boldsymbol{\theta}^{t+1} \leftarrow \boldsymbol{\theta}^{t} +\bigtriangleup\boldsymbol{\theta}^{t}$ 
        \STATE $S^{t+1} \leftarrow S^t e^{-\eta_b \big((\frac{1}{N^t} \sum_{i \in \mathcal{C}^t} b_i^t +\mathcal{N}(0,z_b^2\frac{1}{N^t}^2)) - \kappa \big)}$
        \ENDFOR
        \vfill\null
        \columnbreak
        \STATE \textbf{At client $c_j$:}
        \STATE \textit{ClientUpdate}($\boldsymbol{\theta}^0, c_j, S$):
        \STATE $\boldsymbol{\theta} \leftarrow \boldsymbol{\theta}^0$
        \STATE $\boldsymbol{\theta}_j \leftarrow \boldsymbol{\theta}^0$ (if not initialized)
        \STATE $\mathcal{B} \leftarrow $ batch the client's data $\bD_j$
        \FOR{epoch $e=1,2,...,E$}
        \FOR{$B$ in $\mathcal{B}$}
        \STATE $\boldsymbol{\theta} \leftarrow \boldsymbol{\theta} - \eta \nabla f_j(\boldsymbol{\theta}, B)$
        \STATE $\boldsymbol{\theta}_j \leftarrow \boldsymbol{\theta}_j - \eta_p (\nabla f_j(\boldsymbol{\theta}_j, B) + \lambda_j (\boldsymbol{\theta}_j-\boldsymbol{\theta}^0))$  %
        \ENDFOR
        \ENDFOR
        \STATE $\bigtriangleup \boldsymbol{\theta} \leftarrow \boldsymbol{\theta}-\boldsymbol{\theta}^0$
        \STATE $b \leftarrow \mathbbm{1}_{\Vert \bigtriangleup \boldsymbol{\theta} \Vert_2 \leq S}$
        \STATE return Clip($\bigtriangleup \boldsymbol{\theta}, S), b$ to server
        \STATE
        \STATE
        \STATE
        \STATE \textit{Clip}($\boldsymbol{\theta}, S$):
        \STATE $\quad$ return $\boldsymbol{\theta} \times \frac{S}{\max(\Vert \boldsymbol{\theta} \Vert_2 , S)}$ to client
    \end{algorithmic}
    \end{multicols}
    \vspace{-0.25in}
\end{algorithm*}

First, the notations for the variables used in the algorithm are introduced. The set of $N$ clients $\mathcal{C}$ is split into subsets containing clients grouped according to their desired privacy levels, denoted by $\mathcal{C}_1, \mathcal{C}_2, \mydots, \mathcal{C}_l$. Let the number of clients in the subset $\mathcal{C}_i$ be denoted by $N_i=|\mathcal{C}_i|$. The rest of the hyperparameters in the algorithm are as follows: the noise multipliers $\bz, z_b$, the clipping sensitivity $S$, the learning rate at clients $\eta$, the personalized learning rate at clients $\eta_p$, quantile $\kappa$, and factor $\eta_b$. Also, the superscript $(\cdot)^t$ is used to denote a parameter during the $t$-th training round.

During round $t$ of training, no additional steps are required for the clients during the global model training. Clients train the received model using their local data followed by sending back their clipped updates $\Delta \boldsymbol{\theta}_j^t$ along with their clipping indicator $b_j^t$ to the server. The server collects the updates from clients and performs a \emph{two-step aggregation process}. During the first step, the updates from the clients in each subset $\mathcal{C}_i$ are passed through a $(\epsilon_i, \delta_i)$ DP averaging function to produce $\bigtriangleup\boldsymbol{\tilde{\theta}}^{t}_i$. In the second step of the aggregation, the outputs of the previous averaging functions are combined to produce the next iteration of the model. In this step, the server performs a weighted average of the outputs. The weights for this step are chosen based on the number of clients in each subset in that round, the privacy levels, as well as other parameters. This part resembles the weighted averaging considered in the aforementioned federated linear regression problem in Section \ref{FLR_global}. In general, the goal is to give more weight to updates from clients with less strict privacy requirements compared to the ones with stricter privacy requirements. The output of this step is $\bigtriangleup\boldsymbol{\theta}^{t}$, which is then added to the previous model state to produce the next model state.

To further elaborate on the averaging weights, let us reconsider the simple setup where we have only two subsets of clients, i.e., $\mathcal{C}_1$ and $\mathcal{C}_2$, with DP parameters $(\epsilon_1,\delta_1)$ and $(\epsilon_2,\delta_2)$, respectively. Also, suppose that the second subset has stricter privacy requirements, i.e., $\epsilon_1 \geq \epsilon_2$ and $\delta_1 \geq \delta_2$ The weights $w_1^t$ and $w_2^t$ during round $t$ can be expressed as follows $w_1^t=\frac{N_1^t}{N_1^t + rN_2^t}$ and $w_2^t=\frac{r N_2^t}{N_1^t + rN_2^t}$. In general, we desire the value of the ratio $r$ to be bounded as $0 \leq r \leq 1$ in \textsc{FedHDP} to use the less-private clients' updates more meaningfully. The first factor to consider when choosing $r$ is related to the desired privacy budget, lower privacy budgets require more noise to be added, leading to a lower value of $r$. This intuition was verified in the simplified setting in the previous section. Another factor that is more difficult to quantify is the heterogeneity between the less-private set of clients and the private set of clients. To illustrate this intuition we give the following example. Suppose that the model is being trained on the MNIST dataset where each client has samples of only one digit. Consider two different scenarios: the first is when each of the less-private clients has a digit drawn uniformly from all digits, and the second is when all of the less-private clients have the same digit. It can be argued that the ratio $r$, when every other hyperparameter is fixed, should be higher in the second scenario compared to the first; since contributions from the more-private clients are more significant to the overall model in the second scenario than the first. This will be experimentally verified in the experiments section presented later.

Then, clients need to train personalized models to be used locally. In the \textit{personalization process}, each client simultaneously continues learning a local model when participating in a training round using the local dataset and the most recent version of the global model received during training and the appropriate value of $\lambda$. It is worth noting that the personalization step is similar in spirit to the personalized solution to the federated linear regression problem in Section \ref{FLR_personalized}. Furthermore, the server keeps track of the privacy loss due to the clients' participation in each round by utilizing the moment accountant method \cite{abadi2016deep} for each set of clients to provide them with tighter bounds on their privacy loss.

\section{Experimental Evaluation}\label{experiments}
Thus far, we showed that \textsc{FedHDP} achieves Bayes optimal performance on a class of linear problems. 
In this section, we present the results of a number of more realistic experiments to show the utility gain of the proposed \textsc{FedHDP} algorithm with fine-tuned hyperparameters compared to the baseline \textsc{DP-FedAvg} algorithm, where we additionally apply personalization on this baseline at clients using Ditto~\cite{li2021ditto} and refer to it as \textsc{DP-FedAvg+Ditto}.
Additionally, we compare the performance against another stronger baseline, \textsc{HDP-FedAvg+Ditto}, which is a personalized variant of \textsc{HDP-FedAvg}.
Note that \textsc{HDP-FedAvg+Ditto} can be viewed as a special case of \textsc{FedHDP} with uniform averaging, i.e., $r=1$, instead of a weighted averaging at the server. Note that the hyperparameters for the models, such as learning rate and others, are first tuned using the baseline \textsc{Non-Private} algorithm, then the same values of such hyperparameters are applied in all the private algorithms. The experiments consider the case where two privacy levels are presented to each client to choose from, to be private or non-private. The experiments show that \textsc{FedHDP} outperforms the baseline algorithms with the right choice of the hyperparameters $r, \lambda$ in terms of the global model accuracy, as well as in terms of the average personalized local model accuracy.

\subsection{Setup}
The experiments are conducted on multiple federated datasets, synthetic and realistic. The synthetic datasets are manually created to simulate extreme cases of data heterogeneity often exhibited in FL scenarios. The realistic federated datasets are from TFF \cite{TFF}, where such datasets are assigned to clients according to some criteria. The synthetic dataset is referred to as the non-IID MNIST dataset, and the number of samples at a client is fixed across all clients. Each client is assigned samples randomly from the subsets of samples each with a single digit between $0\!-\!9$. A skewed version of the synthetic dataset is one where non-private clients are sampled from the clients who have the digit $7$ in their data. In the non-IID MNIST dataset, we have $2,000$ clients and we randomly sample $5\%$ of them for training each round. The realistic federated datasets are the FMNIST and FEMNIST from TFF datasets. The FMNIST and FEMNIST datasets contain $3,383$ and $3,400$ clients, respectively, and we sample $\sim 3\%$ of them for training each round. TensorFlow Privacy (TFP) \cite{TFP} is used to compute the privacy loss, i.e., the values of $(\epsilon,\delta)$, incurred during the training phase. Refer to the appendix for an extended description of the used models and their parameters, as well as an extended version of the results.

\textbf{Remark:}  It is worth noting that computing the optimal values of $r$, $\lambda_{\text{np}}$, and $\lambda_{\text{p}}$ for non-convex models such as neural networks is not straightforward. To resolve this issue in practice, we treat them as hyperparameters to be tuned via grid search on the validation set. Note that tuning the regularization parameters $\lambda_{\text{np}}$ and $\lambda_{\text{p}}$ is done locally at clients and, hence, does not lead to any privacy loss. On the other hand, the ratio hyperparameter $r$ needs tuning by the server. The recent work by \cite{papernot2021hyperparameter} shows that tuning hyperparameters from non-private training runs incurs significant privacy loss while such a hyperparameter tuning based on private runs may lead to manageable privacy loss. We must state however that we assumed, as many works in the literature do, that tuning of such hyperparameter does not incur any additional privacy loss on clients. We acknowledge that we have not investigated this privacy loss and leave a comprehensive investigation of this issue to future work.

\subsection{Results}
In this part, we provide the outcomes of the experiments on the datasets mentioned above. In these experiments, we provide results when $5\%$ of the total client population is non-private. Clients that opt out are picked randomly from the set of all clients but fixed for a fair comparison across all experiments. The exception for this assumption is for the skewed non-IID MNIST dataset, where clients that opt out are sampled from the clients who have the digit $7$. All other hyperparameters are fixed. To evaluate the performance of each algorithm, we measure the following quantities for each dataset:

\begin{table*}[t!]
\captionsetup{font=small}
\caption{Summary of the results of experiments on \emph{synthetic datasets}: We compare the performance of the baseline algorithms against \textsc{FedHDP} with tuned hyperparameters. The variance of the performance metric across clients is between parenthesis.}
\vspace{-0.15in}
\label{results}
\begin{center}
\resizebox{\linewidth}{!}{%
\begin{tabular}{|c|c||c||c|c|c||c|c|c|}
\hline
\rowcolor{Gray} \multicolumn{9}{|c|}{nonIID MNIST dataset, $(3.6,10^{-4})$-DP} \\ \hline
\multicolumn{2}{|c||}{Setup} & \multicolumn{4}{|c||}{Global model} & \multicolumn{3}{|c|}{Personalized local models} \\ \hline
\textbf{Algorithm} & \textbf{hyperparameters} &$Acc_{g}\%$ &$Acc_{\text{g},\text{p}}\%$ & $Acc_{\text{g},\text{np}}\%$& $\bigtriangleup_{\text{g}}\%$ &$Acc_{\text{l},\text{p}}\%$ & $Acc_{\text{l},\text{np}}\%$& $\bigtriangleup_{\text{l}}\%$ \\ \hline
\textsc{Non-Private+Ditto} & $\lambda_{\text{np}} \!=\!0.005$ &$93.8$ & - &$93.75 (0.13)$&- &-&$99.98 (0.001)$ &- \\ \hline
\textsc{DP-FedAvg+Ditto} & $\lambda_{\text{p}} \!=\!0.005$ &$88.75$ &$88.64 (0.39)$ & - & - &$99.97 (0.002)$ &- &- \\ \hline
\textsc{HDP-FedAvg+Ditto} & $ \lambda_{\text{p}} \!=\! \lambda_{\text{np}} \!=\!0.005$ &$87.71$ &$87.55(0.42)$ &$88.35(0.34)$&$0.8$ &$99.97(0.001)$&$99.93(0.001)$ &$-0.04$ \\ \hline
\textsc{FedHDP} &\specialcell{$r\!=\!0.01,$ \\ $ \lambda_{\text{p}} \!=\! \lambda_{\text{np}} \!=\!0.005$} & $92.48$ &$92.43(0.30)$ &$93.30(0.21)$&$0.88$ &$99.94(0.001)$&$99.94(0.001)$ &$0.0$\\ \hline

\rowcolor{Gray} \multicolumn{9}{|c|}{Skewed nonIID MNIST dataset, $(3.6,10^{-4})$-DP} \\ \hline
\textsc{Non-Private+Ditto} & $\lambda_{\text{np}} \!=\!0.005$ &$93.67$ & - &$93.62 (0.15)$&- &-&$99.98 (0.001)$ &- \\ \hline
\textsc{DP-FedAvg+Ditto} & $\lambda_{\text{p}} \!=\! 0.005$ &$88.93$ &$88.87 (0.35)$ & - & - &$99.98 (0.001)$ &- &- \\ \hline
\textsc{HDP-FedAvg+Ditto} & $ \lambda_{\text{p}} \!=\! \lambda_{\text{np}} \!=\!0.005$ &$88.25$ &$88.05(0.39)$ &$89.98(0.05)$&$1.93$ &$99.97(0.001)$&$99.85(0.001)$ &$-0.11$ \\ \hline
\textsc{FedHDP} & \specialcell{$r\!=\!0.1,$ \\ $ \lambda_{\text{p}} \!=\! \lambda_{\text{np}} \!=\!0.005$} &$90.36$ &$89.96(0.37)$ &$97.45(0.01)$&$7.49$ &$99.97(0.001)$&$99.76(0.003)$ &$-0.21$  \\ \hline
\textsc{FedHDP} & \specialcell{$r\!=\!0.9,$ \\ $ \lambda_{\text{p}} \!=\! \lambda_{\text{np}} \!=\!0.005$} &$87.96$ &$87.69(0.56)$ &$92.97(0.04)$&$5.28$ &$99.98(0.001)$&$99.96(0.001)$ &$-0.02$  \\ \hline
\end{tabular}
}
\end{center}
\vspace{-.1in}
\end{table*}

\begin{table*}[t!]
\captionsetup{font=small}
\caption{Summarized results of experiments on \emph{realistic federated datasets}:  We compare the performance of the baseline algorithms against \textsc{FedHDP} with tuned hyperparameters. The variance of the performance metric across clients is between parenthesis.}
\label{results2}
\vspace{-0.15in}
\begin{center}
\resizebox{\linewidth}{!}{%
\begin{tabular}{|c|c||c||c|c|c||c|c|c|}
\hline
\rowcolor{Gray} \multicolumn{9}{|c|}{FMNIST dataset, $(0.6,10^{-4})$-DP} \\ \hline
\multicolumn{2}{|c||}{Setup} & \multicolumn{4}{|c||}{Global model} & \multicolumn{3}{|c|}{Personalized local models} \\ \hline
\textbf{Algorithm} & \textbf{hyperparameters} &$Acc_{g}\%$ &$Acc_{\text{g},\text{p}}\%$ & $Acc_{\text{g},\text{np}}\%$& $\bigtriangleup_{\text{g}}\%$ &$Acc_{\text{l},\text{p}}\%$ & $Acc_{\text{l},\text{np}}\%$& $\bigtriangleup_{\text{l}}\%$ \\ \hline
\textsc{Non-Private+Ditto} & $\lambda_{\text{np}} \!=\!0.05$ &$89.65$ & - &$89.35 (1.68)$&- &-&$94.53 (0.59)$ &- \\ \hline
\textsc{DP-FedAvg+Ditto} & $\lambda_{\text{p}} \!=\!0.05$ &$77.61$ &$77.62(2.55)$ & - & - &$90.04 (1.04)$ &- &- \\ \hline
\textsc{HDP-FedAvg+Ditto} & $ \lambda_{\text{p}} \!=\! \lambda_{\text{np}} \!=\!0.005$ &$75.87$ &$75.77(2.84)$ &$74.41(2.8)$&$-1.36$ &$90.45(1.02)$&$92.32(0.8)$ &$1.87$\\ \hline
\textsc{FedHDP} & \specialcell{$r\!=\!0.01,$\\ $\lambda_{\text{p}} =0.05, \lambda_{\text{np}} =0.005$} &$86.88$&$85.36(1.89)$ &$90.02(1.28)$ &$4.66$ &$93.76(0.68)$&$95.94(0.41)$ &$2.18$  \\ \hline

\rowcolor{Gray} \multicolumn{9}{|c|}{FEMNIST dataset, $(4.1,10^{-4})$-DP} \\ \hline
\textsc{Non-Private+Ditto} & $ \lambda_{\text{np}} \!=\!0.25$ &$81.66$ & - &$81.79(1.38)$&- &-&$84.46(0.89)$ &- \\ \hline
\textsc{DP-FedAvg+Ditto} & $\lambda_{\text{p}} \!=\!0.05$ &$75.42$ &$75.86(1.82)$ & - & - &$74.69(1.29)$ &- &- \\ \hline
\textsc{HDP-FedAvg+Ditto} & $ \lambda_{\text{p}} \!=\! \lambda_{\text{np}} \!=\!0.05$ &$75.12$ &$75.87(1.65)$& $78.59(1.58)$& $2.72$& $74.67(1.34)$& $75.95(1.12)$& $1.28$ \\ \hline
\textsc{FedHDP} & \specialcell{$r\!=\!0.1,$\\ $\lambda_{\text{p}}= \lambda_{\text{np}} =0.05$} &$76.52$ &$77.91(1.67)$& $83.9(1.27)$& $5.99$& $77.9(1.22)$& $79.15(0.99)$& $1.25$ \\ \hline
\textsc{FedHDP} & \specialcell{$r\!=\!0.01,$\\ $\lambda_{\text{p}}=\lambda_{\text{np}} =0.25$} &$74.86$&$77.31(2.18)$& $86.73(0.98)$& $9.42$& $81.19(1.02)$& $84.68(0.78)$& $3.49$ \\ \hline

\end{tabular}
}
\end{center}
\vspace{-.15in}
\end{table*}

\begin{enumerate}
\itemsep0em
    \item $Acc_{\text{g}}$: the average test accuracy on the \emph{server} test dataset using the global model.
    \item $Acc_{\text{g},\text{p}}$, $Acc_{\text{g},\text{np}}$: the average test accuracy of all \emph{private} and \emph{non-private} clients using the global model on their local test datasets, respectively.
    \item $Acc_{\text{l},\text{p}}$, $Acc_{\text{l},\text{np}}$: the average test accuracy of all \emph{private} and \emph{non-private} clients using their personalized local models on their local test datasets, respectively.
    \item $\bigtriangleup_{\text{g}}$, $\bigtriangleup_{\text{l}}$: the gain in the average performance of \emph{non-private} clients over the \emph{private} ones using the global model and the personalized local models on their local test datasets, respectively; computed as $\bigtriangleup_{\text{g}}=Acc_{\text{g},\text{np}}-Acc_{\text{g},\text{p}}$ and $\bigtriangleup_{\text{l}}=Acc_{\text{l},\text{np}}-Acc_{\text{l},\text{p}}$.
\end{enumerate}

A summary of the results, shown in Table \ref{results} and Table \ref{results2}, provides the best performance for each experiment along with their corresponding hyperparameters. More detailed results are shown in the appendix. If different values of the hyperparameters in \textsc{FedHDP} yield two competing results, such as one with a better global model performance at the server and one with better personalized models at the clients, we show both.

We can see from Tables \ref{results} and \ref{results2} that \textsc{FedHDP} allows the server to learn better global models while allowing clients to learn better personalized local models compared to the other baselines, i.e., \textsc{DP-FedAvg+Ditto} as well as the \textsc{FedHDP} with $r=1$. For example, the gain due to \textsc{FedHDP} compared to the \textsc{DP-FedAvg+Ditto} in terms of global model performance is up to $9.27\%$. For personalized local models, the gain for clients due to \textsc{FedHDP} compared to \textsc{DP-FedAvg+Ditto} is up to $9.99\%$. Additionally, we can also see the cost in the average performance in personalized local models between clients who choose to opt out of privacy and clients who choose to remain private. This demonstrates the advantage of opting out, which provides clients with an incentive to opt out of DP to improve their personalized local models, for example, non-private clients can gain up to $3.49\%$ on average in terms of personalized local model performance compared to private clients. It is worth mentioning that opting out can also improve the global model's performance on clients' local data. We observe that there is up to $12.4\%$ gain in the average performance of non-private clients in terms of the accuracy of the global model on the local data compared to the one of baseline \textsc{DP-FedAvg+Ditto}. Similar trends can be observed for the other baseline.

\section{Conclusion}\label{conclusion}
In this paper, we considered a new aspect of heterogeneity in FL setups, namely, heterogeneous privacy. We proposed a new setup for privacy heterogeneity between clients where privacy levels are no longer fixed for all clients. In this setup, the clients choose their desired privacy levels according to their preferences and inform the server about the choice.
We provided a formal treatment for the federated linear regression problem and showed the optimality of the proposed solution on the central server as well as the personalized local models in such a setup. Moreover, we have observed that personalization becomes necessary whenever data heterogeneity is present, privacy is required, or both.
We proposed a new algorithm called \textsc{FedHDP} for the considered setup. In \textsc{FedHDP}, the aim is to employ DP to ensure the privacy level desired by each client is met, and we proposed a two-step aggregation scheme at the server to improve the utility of the model. We also utilize personalization to improve the performance at clients.
Finally, we provided a set of experiments on synthetic and realistic federated datasets considering a heterogeneous DP setup. We showed that \textsc{FedHDP} outperforms the baseline private FL algorithm in terms of the global model as well as the personalized local models' performance, and showed the cost of requiring stricter privacy parameters in such scenarios in terms of the gap in the average performance at clients.

\section*{Acknowledgement}
The authors would like to thank Tian Li (Carnegie Mellon University) for helpful discussions and feedback that helped improve this paper.

\bibliographystyle{IEEEtran}
\bibliography{IEEEabrv}

\newpage
\appendix

\section*{Appendix: Organization \& Contents}
We include additional discussions and results complementing the main paper in the appendix.

\begin{itemize}
    \item Appendix~\ref{LR_appendix} includes a discussion on federated linear regression and the optimality of \textsc{FedHDP} along with some related simulation results.
    
    \item Appendix \ref{fpe_appendix} includes additional results for the federated point estimation problem.
    
    \item Finally, we include additional information about the experimental setup along with extended versions of the results of each experiment in Appendix \ref{results_app}.
    
    \item Finally, broader impact and limitations are presented in Appendix \ref{broader_impact}.
\end{itemize}

\section{Theoretical Derivations for Federated Linear Regression}\label{LR_appendix}
In this appendix, we consider applying the proposed algorithm \textsc{FedHDP} on the simplified setup of federated linear regression problem, first introduced by \cite{li2021ditto}. We formally present the setup, solve the optimization problem, and show the Bayes optimality of \textsc{FedHDP} on the global model on the server, as well as the optimality of local models on clients.

\subsection{Federated Linear Regression Setup}\label{LR_setup}
Assume the number of samples per client is fixed and the effect of clipping is negligible. Let us denote the total number of clients as $N$, of whom $N_{1} = \rho_{1} N$ are private with privacy level $p_1$ and $N_{2} = (1-\rho_1)N$ are private with privacy level $p_2$, where $p_2$ is stricter than $p_1$, i.e., more private. The subset of clients with privacy level $p_1$ is denoted by $\mathcal{C}_1$, while the subset of clients with privacy level $p_2$ is denoted by $\mathcal{C}_2$. Denote the number of samples held by each client as $m_s$, the samples at client $c_j$ as $\{\bF_j,\bx_j\}$, where $\bF_j$ is the data features matrix, and $\bx_j$ is the response vector. Let us assume that the features matrix $\bF_j$ has the property of diagonal covariate covariance, i.e., $\bF_j^T\bF_j = n_s I_d$. Let us denote the relationship between $\bx_j$ and $\bF_j$ as
\begin{align}
    \bx_j = \bF_j \boldsymbol{\phi}_j + \bv_j
\end{align}
where elements of the observations noise vector $\bv_j$ are drawn independently from $\mathcal{N}(0,\beta^2)$, and $\boldsymbol{\phi}_j$ is the vector of length $d$ to be estimated. The vector $\boldsymbol{\phi}_j$ is described as
\begin{align}
    \boldsymbol{\phi}_j = \boldsymbol{\phi} + \bp_j
\end{align}
where $\bp_j \sim \mathcal{N}(0, \tau^2 I_d)$, and $\boldsymbol{\phi}$ is the vector to be estimated at the server. It is worth noting that $\tau^2$ is a measure of relatedness, as specified by \cite{li2021ditto}, where larger values of $\tau^2$ reflect increasing data heterogeneity and vice versa. The local loss function at client $c_j$ is as follows
\begin{align}
    f_j(\boldsymbol{\phi})=\frac{1}{m_s} \|\bF_j \boldsymbol{\phi} -\bx_j\|_2^2
\end{align}
Local estimate of $\boldsymbol{\phi}_j$ at client $c_j$ that minimizes the loss function given $\bF_j$ and $\bx_j$ is denoted by $\widehat{\boldsymbol{\phi}}_j$ and is computed as 
\begin{align}
    \widehat{\boldsymbol{\phi}}_j = \big( \bF_j^T\bF_j \big)^{-1} \bF_j^T \bx_j,
\end{align}
which is distributed as $\widehat{\boldsymbol{\phi}}_j\sim \mathcal{N}\big(\boldsymbol{\phi}_j, \beta^2 (\bF_j^T\bF_j)^{-1}\big)$. As assumed earlier, let $\bF_j^T\bF_j = n_s I_d$, then the loss function can be translated to
\begin{align}
    f_j(\boldsymbol{\phi}) = \frac{1}{2}\big\|\boldsymbol{\phi}-\frac{1}{n_s}\sum_{i=1}^{n_s} \bx_{j,i}\big\|_2^2, \label{vector_estimation}
\end{align}
where $\bx_{j,i}$'s are the noisy observations of the vector $\boldsymbol{\phi}_j$ at client $c_j$, which holds due to the assumption of the diagonal covariate covariance matrix.
The updates sent to the server by client $c_j$ are as follows
\begin{align}
    \boldsymbol{\psi}_j = \widehat{\boldsymbol{\phi}}_j + \bl_j
\end{align}

where $\bl_j \sim \mathcal{N}(\boldzero, N_1 \gamma_1^2 I_d)$ for private clients $c_j \in \mathcal{C}_1$ or $\bl_j\sim \mathcal{N}(\boldzero, N_2 \gamma_2^2 I_d)$ for private clients $c_j \in \mathcal{C}_2$. Note that as we mentioned in the paper, we still consider client-level privacy; however, we move the noise addition process from the server side to the client side. This is done such that when the server aggregates the private clients’ updates in each subset the resulting privacy noise variance is equivalent to the desired value by the server, i.e., $\gamma_1^2$ and $\gamma_2^2$. This is done for simplicity and clarity of the discussion and proofs.

In this setup, the problem becomes a vector estimation problem and the goal at the server is to estimate the vector $\boldsymbol{\phi}$ given the updates from all clients, denoted by $\{\boldsymbol{\psi}_i : i\in [N]\}$ as 

\begin{align}
    \boldsymbol{\theta}^* := \arg \min_{\widehat{\boldsymbol{\theta}}} \left\{ \mathbb{E} \left[\left. \frac{1}{2} \|\widehat{\boldsymbol{\theta}} -\boldsymbol{\phi}\|_2^2  \right| \boldsymbol{\psi}_1, \mydots, \boldsymbol{\psi}_N \right] \right\}. \label{global_regression} %
\end{align}
On the other hand, client $c_j$'s goal is to estimate the vector $\boldsymbol{\phi}_j$ given their local estimate $\widehat{\boldsymbol{\phi}}_j$ as well as the updates from all other clients $\{\boldsymbol{\psi}_i : i\in [N]\setminus j\}$ as
\begin{align}
    \hspace{1.2in} \boldsymbol{\theta}^*_j := \arg \min_{\widehat{\boldsymbol{\theta}}} \left\{ \mathbb{E} \left[\left. \frac{1}{2} \| \widehat{\boldsymbol{\theta}}-\boldsymbol{\phi}_j \|_2^2  \right| \{\boldsymbol{\psi}_{i} : i \in  [N]  \setminus j \}, \hat{\boldsymbol{\phi}}_j \right] \right\}.\tag{Local Bayes obj.}\label{local_regression}
\end{align}

Now, considering the value of $\boldsymbol{\phi}$, the covariance matrix of client $c_j$'s update is denoted by $\Sigma_j$. The value of the covariance matrix can be expressed as follows
\begin{align}
    \Sigma_j = \left\{
    \begin{array}{@{}ll@{}}
    \beta^2 (\bF_j^T\bF_j)^{-1} + (\tau^2 +N_1 \gamma_1^2) I_d , & \text{if}\ c_j \in \mathcal{C}_1 \\
    \beta^2 (\bF_j^T\bF_j)^{-1} + (\tau^2 +N_2 \gamma_2^2) I_d , & \text{if}\ c_j \in \mathcal{C}_2
    \end{array}\right.
\end{align}
We have $\bF_j^T\bF_j = n_s I_d$, and let $\frac{\beta^2}{n_s}=\alpha^2$, $\sigma_c^2 = \alpha^2+\tau^2$, $\sigma_{p_1}^2 = \sigma_c^2 + N_1 \gamma_1^2$, and $\sigma_{p_2}^2 = \sigma_c^2 + N_2 \gamma_2^2$, then we have
\begin{align}
    \Sigma_j &= \left\{
    \begin{array}{@{}ll@{}}
    (\alpha^2 + \tau^2 +N_1 \gamma_1^2) I_d , & \text{if}\ c_j \in \mathcal{C}_1 \\
    (\alpha^2+\tau^2 +N_2 \gamma_2^2) I_d , & \text{if}\ c_j \in \mathcal{C}_2
    \end{array}\right.\\
    & = \left\{
    \begin{array}{@{}ll@{}}
    \sigma_{p_1}^2 I_d, \hspace{.32in}& \text{if}\ c_j \in \mathcal{C}_1 \\
    \sigma_{p_2}^2 I_d, & \text{if}\ c_j \in \mathcal{C}_2
    \end{array}\right.
\end{align}

Next, we discuss the optimality of \textsc{FedHDP} for the specified federated linear regression problem for the server's global model as well as the clients' personalized local models.

\subsection{Global Estimate on The Server}\label{LR_global}
In the considered federated setup, the server aims to find $\widehat{\boldsymbol{\theta}}^*$ described as follows
\begin{align}
    \widehat{\boldsymbol{\theta}}^* := \arg \min_{\widehat{\boldsymbol{\theta}}} \left\{   \frac{1}{2} \left\|\sum_{i \in [N]} w_i \boldsymbol{\psi}_i-\widehat{\boldsymbol{\theta}}  \right\|_2^2  \right\}. \label{global_objective_opt} %
\end{align}
The server's goal, in general, is to combine the client updates such that the estimation error of $\boldsymbol{\phi}$ in \eqref{global_regression} is minimized, i.e., the server aims to find the Bayes optimal solution described in \eqref{global_regression}. For the considered setup, the server aims to utilize the updates sent by clients, i.e., $\{\boldsymbol{\psi}_i : i\in [N]\}$, to estimate the vector $\boldsymbol{\phi}$. The estimate at the server is denoted by $\boldsymbol{\theta}$. Our goal in this part is to show that the solution to \eqref{global_objective_opt}, which is the solution in \textsc{FedHDP} algorithm converges to the Bayes optimal solution in \eqref{global_regression}. First, we state an important lemma that will be used throughout this section.

\begin{lemma}[Lemma 2 in \cite{li2021ditto}]\label{lemma2_ditto}
Let $\boldsymbol{\phi}$ be drawn from the non-informative uniform prior on $\mathbb{R}^d$. Also, let $\{\boldsymbol{\psi}_i : i\in [N]\}$ denote noisy observations of $\boldsymbol{\phi}$  with independent additive zero-mean independent Gaussian noise and corresponding covariance matrices $\{\Sigma_i : i\in [N]\}$. Let
\begin{align}
    \Sigma_{\boldsymbol{\phi}} = \Big( \sum_{i \in [N]} \Sigma_i^{-1} \Big)^{-1}.
\end{align}
Then, conditioned on $\{\boldsymbol{\psi}_i : i\in [N]\}$, we have
\begin{align}
    \boldsymbol{\phi} = \Sigma_{\boldsymbol{\phi}} \sum_{i \in [N]} \Sigma_i^{-1} \boldsymbol{\psi}_i + \bp_{\boldsymbol{\phi}},
\end{align}
where $\bp_{\boldsymbol{\phi}} \sim \mathcal{N}(\boldzero,\Sigma_{\boldsymbol{\phi}})$, which is independent of $\{\boldsymbol{\psi}_i : i\in [N]\}$.
\end{lemma}

Next, we state the Bayes optimality of the solution at the server.

\begin{lemma}[Global estimate optimality]\label{opt_global_estimate}
The proposed solution, from the server's point of view, with weights $w_j$'s chosen below, is Bayes optimal in the considered federated linear regression problem.
\begin{align}
    w_j = \left\{
    \begin{array}{@{}ll@{}}
    \frac{1}{N_1+N_2 r^*} , & \text{if}\ c_j \in \mathcal{C}_1 \\
    \frac{r^*}{N_1+N_2 r^*} , & \text{if}\ c_j \in \mathcal{C}_2
    \end{array}\right.
\end{align}
where
\begin{align}
    r^*=\frac{\sigma_c^2+N_1 \gamma_1^2}{\sigma_c^2 + N_2 \gamma_2^2}.\label{opt_ratio_LR}
\end{align}
Furthermore, the covariance of the estimation error is:
\begin{align}
    \Sigma_{s,\emph{opt}} = \frac{1}{N} \left[ \frac{(\sigma_c^2+N_1 \gamma_1^2) (\sigma_c^2 + N_2 \gamma_2^2)}{\sigma_c^2 + (1-\rho_1) N_1 \gamma_1^2 + \rho_1 N_2 \gamma_2^2} \right] I_d. \label{opt_covar_LR}
\end{align}
\end{lemma}
\begin{proof}
First, for the considered setup, Lemma \ref{lemma2_ditto} states that the optimal aggregator at the server is the weighted average of the client updates. The server observes the updates $\{\boldsymbol{\psi}_i : i\in [N]\}$, which are noisy observations of $\boldsymbol{\phi}$ with zero-mean Gaussian noise with corresponding covariance matrices $\{\Sigma_i : i\in [N]\}$. Then, the server computes its estimate $\boldsymbol{\theta}$ of $\boldsymbol{\phi}$ as
\begin{align}
    \boldsymbol{\theta} &= \Sigma_{\boldsymbol{\theta}} \sum_{i \in [N]} \Sigma_i^{-1} \boldsymbol{\psi}_i + \bp_{\boldsymbol{\theta}},
\end{align}
where $\bp_{\boldsymbol{\theta}} \sim \mathcal{N} (\boldzero, \Sigma_{\boldsymbol{\theta}})$ and
\begin{align}
    \Sigma_{\boldsymbol{\theta}} = \Big( \sum_{i \in [N]} \Sigma_i^{-1} \Big)^{-1} &= \Big( N_1(\sigma_c^2 + N_1 \gamma_1^2)^{-1} I_d + N_2 (\sigma_c^2+ N_2 \gamma_2^2)^{-1} I_d \Big)^{-1}\\
    & = \frac{1}{N} \left[ \frac{(\sigma_c^2+N_1 \gamma_1^2) (\sigma_c^2 + N_2 \gamma_2^2)}{\sigma_c^2 + (1-\rho_1) N_1 \gamma_1^2 + \rho_1 N_2 \gamma_2^2} \right] I_d. \label{opt_var_LR_original}
\end{align}
In \textsc{FedHDP} with only two subsets of clients, we only have a single hyperparameter to manipulate server-side, which is the ratio $r$ that is the ratio of the weight dedicated for clients with higher privacy level to the one for clients with the lower privacy level. To achieve the same noise variance as in \eqref{opt_var_LR_original} we need to choose the ratio $r$ carefully. To this end, setting $r=\frac{\sigma_c^2 + N_1 \gamma_1^2}{\sigma_c^2 + N_2 \gamma_2^2}$ in \textsc{FedHDP} results in additive noise variance in the estimate with zero mean and covariance matrix as follows
\begin{align}
    \Sigma_{s,\emph{opt}} = \frac{1}{N} \left[ \frac{(\sigma_c^2+N_1 \gamma_1^2) (\sigma_c^2 + N_2 \gamma_2^2)}{\sigma_c^2 + (1-\rho_1) N_1 \gamma_1^2 + \rho_1 N_2 \gamma_2^2} \right] I_d.
\end{align}

Therefore, the weighted average of the updates using the above weights, results in the solution being Bayes optimal, i.e., produces $\boldsymbol{\theta}^*$.
\end{proof}

\subsection{Personalized Local Estimates on Clients}\label{LR_local}
As mentioned in Section \ref{FLR_personalized}, \textsc{FedHDP} differs from \textsc{Ditto} in many ways. First, the global model aggregation is different, i.e., \textsc{FedAvg} was employed in \textsc{Ditto} compared to the 2-step aggregator in \textsc{FedHDP}. Second, in \textsc{Ditto} measuring the performance only considers benign clients, while in \textsc{FedHDP} it is important to measure the performance of all subsets of clients, and enhancing it across all clients is desired. In this part, we focus on the personalization part for both sets of clients. The goal at clients is to find the Bayes optimal solution to the \eqref{local_regression}. However, in the considered federated setup, clients don't have access to individual updates from other clients, but rather have the global estimate $\widehat{\boldsymbol{\theta}}^*$. So, we have the local \textsc{FedHDP} objective as
\begin{align}
    \hspace{1.2in} \widehat{\boldsymbol{\theta}}^*_j := \arg \min_{\widehat{\boldsymbol{\theta}}} \left\{ \frac{1}{2} \|\widehat{\boldsymbol{\theta}} - \widehat{\boldsymbol{\phi}}_j \|_2^2  +  \frac{\lambda}{2} \|\widehat{\boldsymbol{\theta}}-\widehat{\boldsymbol{\theta}}^*\|_2^2   \right\}.\tag{Local FedHDP obj.}\label{local_objective_opt}
\end{align}

First, we compute the Bayes optimal local estimate $\boldsymbol{\theta}_j^*$ of $\boldsymbol{\phi}_j$ for the local objective at client $c_j$. We consider client $c_j$, which can be either in $\mathcal{C}_1$ or $\mathcal{C}_2$, and compute their minimizer of \eqref{local_regression}. In this case, the client is given all other clients' estimates $\{\boldsymbol{\psi}_i : i\in [N] \setminus j\}$ and has their own local estimate $\hat{\boldsymbol{\phi}}_j$. To this end, we utilize Lemma \ref{lemma2_ditto} to find the optimal estimate $\boldsymbol{\theta}_j^*$. Given the updates by all other clients $\{\boldsymbol{\psi}_i : i\in [N] \setminus j\}$, the client can compute the estimate $\boldsymbol{\phi}^{\setminus j}$ of the value of $\boldsymbol{\phi}$ as
\begin{align}
    \boldsymbol{\phi}^{\setminus j} = \Sigma_{\boldsymbol{\phi}^{\setminus j}}\Big( \sum_{i=[N] \setminus j} \Sigma_i^{-1} \boldsymbol{\psi}_i \Big)+ \bp_{\boldsymbol{\phi}^{\setminus j}},
\end{align}
where $\bp_{\boldsymbol{\phi}^{\setminus j}} \sim \mathcal{N}(\boldzero, \Sigma_{\boldsymbol{\phi}^{\setminus j}})$ and
\begin{align}
    \Sigma_{\boldsymbol{\phi}^{\setminus j}} &= \Big( \sum_{i=[N] \setminus j} \Sigma_i^{-1} \Big)^{-1},\\
    &= \Big( m \frac{1}{\sigma_{p_1}^2}I_d  + n \frac{1}{\sigma_{p_2}^2} I_d \Big)^{-1},\\
    &= \frac{\sigma_{p_1}^2 \sigma_{p_2}^2}{n \sigma_{p_1}^2 + m \sigma_{p_2}^2} I_d,
\end{align}
where $n=N_2-1, m=N_1$ if $c_j \in \mathcal{C}_2$, or $n=N_2, m=N_1-1$ if $c_j \in \mathcal{C}_1$. Then, the client uses $\Sigma_{\boldsymbol{\phi}^{\setminus j}}$ and $\hat{\boldsymbol{\phi}}_j$ to estimate $\boldsymbol{\theta}_{j}^*$ as
\begin{align}
    \boldsymbol{\theta}_{j}^* &= \Sigma_{\boldsymbol{\theta}_{j}^*}\Big( (\Sigma_{\boldsymbol{\phi}^{\setminus j}}+\tau^2 I_d)^{-1} \boldsymbol{\phi}^{\setminus j} + (\sigma_c^2-\tau^2)^{-1} \hat{\boldsymbol{\phi}}_j \Big)+\bp_{\boldsymbol{\theta}_{j}^*} ,\\
    &= \Sigma_{\boldsymbol{\theta}_{j}^*}\bigg( \Big( \frac{n \sigma_{p_1}^2 + m \sigma_{p_2}^2}{\sigma_{p_1}^2 \sigma_{p_2}^2 + \tau^2(n \sigma_{p_1}^2 + m \sigma_{p_2}^2)} \Big) \boldsymbol{\phi}^{\setminus j} + \frac{1}{\sigma_c^2-\tau^2} \hat{\boldsymbol{\phi}}_j^* \bigg)+\bp_{\boldsymbol{\theta}_{j}},\label{sigma_j_estimate}
\end{align}
where $\bp_{\boldsymbol{\theta}_j^*} \sim \mathcal{N}(\boldzero, \Sigma_{\boldsymbol{\theta}_{j}^*})$ and
\begin{align}
    \Sigma_{\boldsymbol{\theta}_{j}^*} &= \bigg( \Big( \frac{\sigma_{p_1}^2 \sigma_{p_2}^2 + \tau^2(n \sigma_{p_1}^2 + m \sigma_{p_2}^2)}{n \sigma_{p_1}^2 + m \sigma_{p_2}^2} I_d \Big)^{-1} + ((\sigma_c^2-\tau^2) I_d)^{-1} \bigg)^{-1},\\
    &= \frac{(\sigma_c^2-\tau^2) \big(\sigma_{p_1}^2\sigma_{p_2}^2+\tau^2(n \sigma_{p_1}^2 + m \sigma_{p_2}^2)\big)}{\sigma_c^2\big(n \sigma_{p_1}^2 + m \sigma_{p_2}^2 \big) + \sigma_{p_1}^2 \sigma_{p_2}^2 } I_d.
\end{align}
We expand \eqref{sigma_j_estimate} as
\begin{align}
    \boldsymbol{\theta}_{j}^* = \frac{\sigma_{p_1}^2 \sigma_{p_2}^2 + \tau^2(n \sigma_{p_1}^2 + m \sigma_{p_2}^2)}{\sigma_c^2(n \sigma_{p_1}^2 + m \sigma_{p_2}^2)+\sigma_{p_1}^2 \sigma_{p_2}^2}\hat{\boldsymbol{\phi}}_j &+ \frac{\sigma_{p_2}^2(\sigma_c^2 -\tau^2) }{\sigma_c^2(n \sigma_{p_1}^2 + m \sigma_{p_2}^2)+\sigma_{p_1}^2 \sigma_{p_2}^2} \sum_{\substack{c_i\in \mathcal{C}_1\\i \neq j}} \boldsymbol{\psi}_i \nonumber \\
    &+ \frac{\sigma_{p_1}^2(\sigma_c^2 -\tau^2) }{\sigma_c^2(n \sigma_{p_1}^2 + m \sigma_{p_2}^2)+\sigma_{p_1}^2 \sigma_{p_2}^2} \sum_{\substack{c_i\in \mathcal{C}_2\\i \neq j}} \boldsymbol{\psi}_i + \bp_{\boldsymbol{\theta}_{j}^*}. \label{minimizer_theta}
\end{align}
This is the Bayes optimal solution to the local Bayes objective optimization problem for client $c_j$ in \eqref{local_regression}. Now, recall that in \textsc{FedHDP}, the clients do not have access to individual client updates, but rather the global model. Therefore, the clients solve the \textsc{FedHDP} local objective in \eqref{local_objective_opt}. Given a value of $\lambda_j$ and the global estimate $\hat{\boldsymbol{\theta}}^*$, the minimizer $\hat{\boldsymbol{\theta}}_j(\lambda_j)$ of \eqref{local_objective_opt} is
\begin{align}
    \hat{\boldsymbol{\theta}}_j(\lambda_j) &= \frac{1}{1+\lambda_j} \Big( \hat{\boldsymbol{\phi}}_j+\lambda_j \hat{\boldsymbol{\theta}}^* \Big)\label{opt_local_model_lambda}\\
    &=\frac{1}{1+ \lambda_j} \bigg( 
    \frac{(N_1+N_2 r)+\lambda_j i_j}{ (N_1+N_2 r)}\hat{\boldsymbol{\phi}}_j + \frac{\lambda_j}{(N_1+N_2 r)} \sum_{\substack{c_i\in \mathcal{C}_1\\i \neq j}} \boldsymbol{\psi}_i+ \frac{\lambda_j r}{(N_1+N_2 r)} \sum_{\substack{c_i\in \mathcal{C}_2 \\i \neq j}} \boldsymbol{\psi}_i \bigg), \label{minimizer_lambda}
\end{align}
where $i_j = 1$ if $c_j \in \mathcal{C}_1$ or $i_j=r$ if $c_j \in \mathcal{C}_2$. Now, we are ready to state the Bayes optimality of the local FedHDP objective for optimal values $\lambda_j^*$ for all clients.

\begin{lemma}[Local estimates optimality]\label{opt_personalized_estimate}
The solution to the local FedHDP objective from the clients' point of view using $\lambda_j^*$ chosen below, under the assumption of global estimate optimality stated in Lemma \ref{opt_global_estimate}, is Bayes optimal in the considered federated linear regression problem.
\begin{align}
    \lambda_j^* = \left\{
    \begin{array}{@{}ll@{}}
    \frac{N (1+\Upsilon^2) + N_1 \Gamma_2^2 + N_2 \Gamma_1^2}{N\Upsilon^2(1+\Upsilon^2)+\Upsilon^2 ((N_2+1) \Gamma_1^2 + N_1 \Gamma_2^2)+\Gamma_1^2(1+\Gamma_2^2)}, & \emph{if}\ c_j \in \mathcal{C}_1 \\
    \frac{N (1+\Upsilon^2) + N_1 \Gamma_2^2 + N_2 \Gamma_1^2} {N \Upsilon^2 (1 + \Upsilon^2) +\Upsilon^2 ( N_2 \Gamma_1^2 + (N_1+1) \Gamma_2^2) + \Gamma_2^2 (1 + \Gamma_1^2)}, & \emph{if}\ c_j \in \mathcal{C}_2
    \end{array}\right.\label{opt_lambda}
\end{align}
where $\Upsilon^2 = \frac{\tau^2}{\alpha^2}$, $\Gamma_1^2 = \frac{N_1 \gamma_1^2}{\alpha^2}$, and $\Gamma_2^2 = \frac{N_2 \gamma_2^2}{\alpha^2}$.
\end{lemma}
\begin{proof}
To prove this lemma, as shown in \cite{li2021ditto}, we only need to find the optimal values of $\lambda_j^*$ that minimize the following
\begin{align}
    \lambda_j^* = \arg \min_{\lambda}\mathbb{E}\big(\|\boldsymbol{\theta}_j^*-\hat{\boldsymbol{\theta}}_j(\lambda) \|_2^2 \big| \boldsymbol{\phi}^{\setminus j}, \hat{\boldsymbol{\phi}}_j \big)
\end{align}
for private and non-private clients. To compute the values of $\lambda_j^*$, we plug in the values of $\boldsymbol{\theta}_j^*$ from \eqref{minimizer_theta} and $\boldsymbol{\theta}_j(\lambda)$ in \eqref{minimizer_lambda}, which gives us the following
\begin{gather}
    \lambda_1 = \frac{(N_1+ N_2 r) \big( \sigma_c^2 (n \sigma_{p_1}^2 + m \sigma_{p_2}^2) - \tau^2 (n \sigma_{p_1}^2 + m \sigma_{p_2}^2) \big)}{(N_1+ N_2 r) \big( \tau^2 (n \sigma_{p_1}^2 + m \sigma_{p_2}^2) + \sigma_{p_1}^2 \sigma_{p_2}^2 \big) - i_j \big(\sigma_c^2 (n \sigma_c^2 + m \sigma_p^2) + \sigma_{p_1}^2 \sigma_{p_2}^2\big) },\\
    \lambda_2 = \frac{(N_1+ N_2 r) (\sigma_c^2-\tau^2) \sigma_{p_2}^2}{\sigma_c^2 (n \sigma_{p_1}^2 + m \sigma_{p_2}^2) + \sigma_{p_1}^2 \sigma_{p_2}^2-(N_1+ N_2 r)(\sigma_c^2-\tau^2)\sigma_{p_2}^2},\\
    \lambda_3 = \frac{(N_1+ N_2 r) (\sigma_c^2-\tau^2) \sigma_{p_2}^2}{\sigma_c^2 (n \sigma_{p_1}^2 + m \sigma_{p_2}^2) + \sigma_{p_1}^2 \sigma_{p_2}^2-(N_1+ N_2 r)(\sigma_c^2-\tau^2)\sigma_{p_2}^2},\\
    \text{and\:\:} \lambda_j^* = \frac{1}{3}(\lambda_1 + \lambda_2 + \lambda_3) \label{opt_lambda_general}
\end{gather}
where $r=\frac{\sigma_{p_1}^2}{\sigma_{p_2}^2}$. For client $c_j \in \mathcal{C}_1$, we have $i_j = 1, n = N_2$ and $m = N_1-1$. Setting $\Upsilon^2 = \frac{\tau^2}{\alpha^2}$, $\Gamma_1^2 = \frac{N_1 \gamma_1^2}{\alpha^2}$, and $\Gamma_2^2 = \frac{N_2 \gamma_2^2}{\alpha^2}$ and substituting in \eqref{opt_lambda_general} gives the desired result in \eqref{opt_lambda}. For client $c_j \in \mathcal{C}_2$, we have $i_j = r, n = N_2-1$ and $m = N_1$. Setting $\Upsilon^2 = \frac{\tau^2}{\alpha^2}$, $\Gamma_1^2 = \frac{N_1 \gamma_1^2}{\alpha^2}$, and $\Gamma_2^2 = \frac{N_2 \gamma_2^2}{\alpha^2}$ and substituting in \eqref{opt_lambda_general} gives the desired results in \eqref{opt_lambda}. As a result, the resulting $\hat{\boldsymbol{\theta}}_j(\lambda_j^*)$ is Bayes optimal.
\end{proof}

Next, we provide a few examples of corner cases for both $\lambda_{\text{p}}^*$ and $\lambda_{\text{np}}^*$ for the considered linear regression setup:
\begin{itemize}
    \item $r \rightarrow 1$, i.e., noise added for privacy is similar for both sets of clients $\Gamma_1^2 \rightarrow \Gamma_2^2$, then we have $\lambda_1^* \rightarrow \frac{N(1+\Upsilon^2)+N\Gamma_2^2}{N \Upsilon^2(1+\Upsilon^2) + (N+1) \Upsilon^2 \Gamma_2^2 + (1+\Gamma_2^2)\Gamma_2^2}$ and $\lambda_2^* \rightarrow \frac{N(1+\Upsilon^2)+N\Gamma_2^2}{N \Upsilon^2(1+\Upsilon^2) + (N+1) \Upsilon^2 \Gamma_2^2 + (1+\Gamma_2^2)\Gamma_2^2}$. If $\Gamma_1^2 = \Gamma_2^2 = 0$, then we have $\lambda_1^* = \lambda_2^* \rightarrow \frac{1}{\Upsilon^2}$ as in \textsc{Ditto} with \textsc{FedAvg} and no malicious clients.
    \item $N_2 \rightarrow N$, i.e., all clients have the same privacy level, as in \textsc{DP-FedAvg}, $\lambda_2^* \rightarrow  \frac{N}{\Upsilon^2 N + \Gamma_2^2}$.
    \item $\alpha^2 \rightarrow 0$, $\lambda_2^* \rightarrow 0$ and $\lambda_1^* \rightarrow 0$. The optimal estimator for all clients approaches the local estimator, i.e., $\hat{\boldsymbol{\theta}}_j(\lambda_j^*) \rightarrow \hat{\boldsymbol{\phi}}_j$.
    \item $\tau^2 \rightarrow 0$, i.e., all clients have IID samples, $\lambda_1^* \rightarrow \frac{N+N_2 \Gamma_1^2 + N_1 \Gamma_2^2}{\Gamma_1^2(1+\Gamma_2^2)}$ and $\lambda_2^* \rightarrow \frac{N+N_2 \Gamma_1^2 + N_1 \Gamma_2^2}{\Gamma_2^2(1+\Gamma_1^2)}$.
\end{itemize}

\subsection{Optimality of \textsc{FedHDP}}
Next, we show the convergence of the \textsc{FedHDP} algorithm to the \textsc{FedHDP} global and local objectives for the linear regression problem described above as follows
\begin{lemma}[FedHDP convergence]\label{convergence_LR}
FedHDP, with learning rate $\eta=1$ and $\eta_p=\frac{1}{1+\lambda_j}$ converges to the global FedHDP objective and the local FedHDP objective.
\end{lemma}
\begin{proof}
In the considered setup, we denote $\hat{\boldsymbol{\phi}}_j=\frac{1}{n_s}\sum_{i=1}^{n_s} \bx_{j,i}$ at client $c_j$. The client updates the global estimation $\boldsymbol{\theta}$ by minimizing the loss function in \eqref{vector_estimation}. The global estimation update at the client follows
\begin{align}
    \boldsymbol{\theta} \leftarrow \boldsymbol{\theta} - \eta (\boldsymbol{\theta}-\hat{\boldsymbol{\phi}}_j).
\end{align}
Updating the estimation once with $\eta=1$ results in the global estimation update being $\hat{\boldsymbol{\phi}}_j$, adding the noise results in the same $\boldsymbol{\psi}_j$, and hence the global estimate in the next iteration is unchanged.
As for the local FedHDP estimation, when the client receives the global estimate $\boldsymbol{\theta}$ after the first round, the client updates its estimate $\boldsymbol{\theta}_j$ as
\begin{align}
    \boldsymbol{\theta}_j \leftarrow \boldsymbol{\theta}_j - \eta_p \big( (\boldsymbol{\theta}_j-\hat{\boldsymbol{\phi}}_j)+\lambda_j (\boldsymbol{\theta}_j-\boldsymbol{\theta}) \big).
\end{align}
Updating the estimate once with $\eta_p=\frac{1}{1+\lambda_j}$ gives $\boldsymbol{\theta}_j = \frac{1}{1+\lambda_j}(\hat{\boldsymbol{\phi}}_j+\lambda_j \boldsymbol{\theta})$, which is the solution to the local FedHDP objective in \eqref{opt_local_model_lambda}. Hence, FedHDP converges to the global and local FedHDP objectives.
\end{proof}

Next, we state the optimality theorem of \textsc{FedHDP} algorithm for the considered setup described above.
\begin{theorem}[FedHDP optimality]\label{FedHDP_optimality}
FedHDP from the server's point of view with ratio $r^*$ chosen below is Bayes optimal (i.e., $\boldsymbol{\theta}$ converges to $\boldsymbol{\theta}^*$) in the considered federated linear regression
problem.
\begin{align}
    r^*=\frac{\sigma_c^2+N_1 \gamma_1^2}{\sigma_c^2 + N_2 \gamma_2^2}.
\end{align}
Furthermore, FedHDP from the client's point of view, with $\lambda_j^*$ chosen below, is Bayes optimal (i.e., $\boldsymbol{\theta}_j$ converges to $\boldsymbol{\theta}^*_j$ for each client $j \in [N]$) in the considered federated linear regression problem.
\begin{align}
    \lambda_j^* = \left\{
    \begin{array}{@{}ll@{}}
    \frac{N (1+\Upsilon^2) + N_1 \Gamma_2^2 + N_2 \Gamma_1^2}{N\Upsilon^2(1+\Upsilon^2)+\Upsilon^2 ((N_2+1) \Gamma_1^2 + N_1 \Gamma_2^2)+\Gamma_1^2(1+\Gamma_2^2)}, & \emph{if}\ c_j \in \mathcal{C}_1 \\
    \frac{N (1+\Upsilon^2) + N_1 \Gamma_2^2 + N_2 \Gamma_1^2} {N \Upsilon^2 (1 + \Upsilon^2) +\Upsilon^2 ( N_2 \Gamma_1^2 + (N_1+1) \Gamma_2^2) + \Gamma_2^2 (1 + \Gamma_1^2)}, & \emph{if}\ c_j \in \mathcal{C}_2
    \end{array}\right.
\end{align}
\end{theorem}
\begin{proof}
Follows by observing Lemma \ref{convergence_LR}, which states that the algorithm converges to the global and local FedHDP objectives, then by Lemma \ref{opt_global_estimate} and Lemma \ref{opt_personalized_estimate}, which state that the solution to the FedHDP objective is the Bayes optimal solution for both global and local objectives.
\end{proof}

\begin{figure}[t!]
    \centering
    \captionsetup{font=small}
    \includegraphics[height=2.8in]{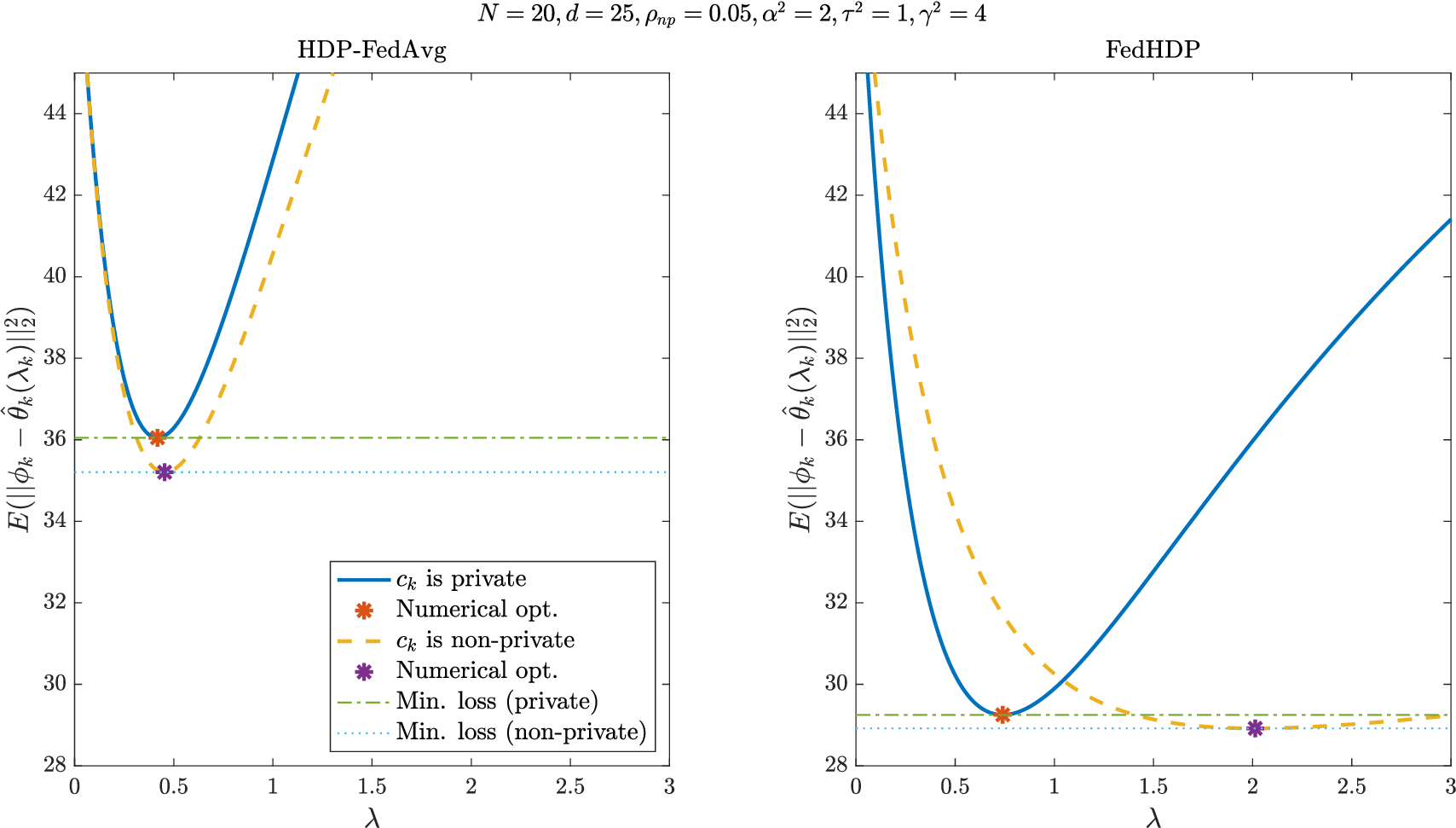} 
    \caption{The effect of opting out on the personalized local model estimate for a linear regression problem as a function of $\lambda$ when employing (left) \textsc{HDP-FedAvg} and (right) \textsc{FedHDP}.}
    \label{gain_opting_out_LR}
\end{figure}

\subsection{Privacy-Utility Tradeoff}\label{LR_tradeoff}
Let us consider the special case of opt-out of privacy in this simplified linear regression setup, i.e., $\gamma_1^2 = 0$ and $\gamma_2^2 = \gamma^2$. We would like to observe the effect of opting out of privacy on the client's personalized local model, compared to the one where the client remains private. We show an experiment comparing \textsc{FedHDP} using $r^*$ against \textsc{HDP-FedAvg} for two scenarios. The first is when the client chooses to opt out of privacy, and the second is when the client chooses to remain private. See Figure \ref{gain_opting_out_LR} for the results of such experiment. We can see that \textsc{FedHDP} outperforms the one with \textsc{HDP-FedAvg}, and the cost of remaining private is evident in terms of higher loss at the client.

\subsection{Extension Beyond Two Privacy Levels}
\label{app:more-than-two}
The setup for federated learning with two privacy levels was presented to show the steps and the explicit expressions for the values of the ratio and the regularization hyperparameters. Next, we show briefly that the same derivation can be extended to find the explicit expressions for the hyperparameters in a general case of federated linear regression with clients choosing one of $l$ privacy levels.  To start, assume that we have $l$ privacy levels where clients can be split into $l$ subsets denoted by $\mathcal{C}_i$ for $i=1,2,\mydots,l$, each has $N_i>1$ clients, respectively, while other notations are still the same. Notice that this setup contains the most general case where $l = |\mathcal{C}|$ and each client has their own privacy level.

Similar to the setup with two privacy levels, each client sends their update to the server after adding the appropriate amount of Gaussian noise, i.e., $\mathcal{N}(\boldzero, N_i\gamma_i^2 I_d)$ for client $c_j \in \mathcal{C}_i$, for privacy. Let us denote the updates sent to the server by $\{\boldsymbol{\psi} : i\in [N]\}$ which are estimates of $\boldsymbol{\phi}$ with zero-mean Gaussian noise with corresponding covariance matrices $\{\Sigma_i : i\in [N]\}$, where $\Sigma_j$ for client $c_j \in \mathcal{C}_i$ is expressed as

\begin{align}
    \Sigma_j &= (\alpha^2 + \tau^2 +N_i \gamma_i^2) I_d = \sigma_j^2 I_d = \sigma_{p_i}^2 I_d.
\end{align}

Then, the server computes its estimate $\boldsymbol{\theta}$ of $\boldsymbol{\phi}$ as
\begin{align}
    \boldsymbol{\theta} &= \Sigma_{\boldsymbol{\theta}} \sum_{i \in [N]} \Sigma_i^{-1} \boldsymbol{\psi}_i + \bp_{\boldsymbol{\theta}},
\end{align}

where $\bp_{\boldsymbol{\theta}} \sim \mathcal{N}(\boldzero, \Sigma_{\theta})$ and
\begin{align}
    \Sigma_{\boldsymbol{\theta}} = \Big( \sum_{i \in [N]} \Sigma_i^{-1} \Big)^{-1} = \Big( \frac{\prod_{i\in [l]} \sigma_{p_i}^2}{\sum_{i\in [l]} N_i \prod_{k \in [l] \setminus i} \sigma_{p_k}^2} \Big) I_d.
\end{align}

In \textsc{FedHDP}, the server applies a weighted averaging to the clients' updates $\boldsymbol{\psi}_i$'s of this form
\begin{align}
    \boldsymbol{\theta} = \sum_{i \in [N]} w_i \boldsymbol{\psi}_i.
\end{align}
To achieve the optimal covariance of the estimation at the server, the resulting weights used for client $c_j$ in $\mathcal{C}_i$ at the server as follows
\begin{align}
    w_j = \frac{r_i}{\sum_{k=[l]} N_k r_k}.
\end{align}
In this case, we have $l$ ratio hyperparameters $r_i$'s to tune. Similar to the approach followed for the 2-level privacy heterogeneity, we can find the optimal values of the ratio hyperparameters that achieve the optimal covariance of the estimation at the server. The optimal values of $r_i^*$ are
\begin{align}
    r_i^* = \frac{\sigma_{p_1}^2}{\sigma_{p_i}^2}.
\end{align}
Next, we compute the Bayes optimal local estimate $\boldsymbol{\theta}_j^*$ of $\boldsymbol{\phi}_j$ for the local objective at client $c_j$. We consider client $c_j$, which can be in any private set of clients $\mathcal{C}_i$, and compute their minimizer of \eqref{local_regression}. In this case, the client has access to all other clients' private estimates $\{\boldsymbol{\psi}_i : i \in [N] \setminus j\}$ and has their own local non-private estimate $\hat{\boldsymbol{\phi}}_j$. Similar to the approach before, we utilize Lemma \ref{lemma2_ditto} to find the optimal estimate $\boldsymbol{\theta}_j^*$. Given the updates by all other clients $\{\boldsymbol{\psi}_i : i\in [N] \setminus j\}$, the client can compute the estimate $\boldsymbol{\phi}^{\setminus j}$ of the value of $\boldsymbol{\phi}$ as
\begin{align}
    \boldsymbol{\phi}^{\setminus j} = \Sigma_{\boldsymbol{\phi}^{\setminus j}}\Big( \sum_{i=[N] \setminus j} \Sigma_i^{-1} \boldsymbol{\psi}_i \Big)+ \bp_{\boldsymbol{\phi}^{\setminus j}},
\end{align}
where $\bp_{\boldsymbol{\phi}^{\setminus j}} \sim \mathcal{N}(\boldzero, \Sigma_{\boldsymbol{\phi}^{\setminus j}})$ and
\begin{align}
    \Sigma_{\boldsymbol{\phi}^{\setminus j}} &= \Big( \sum_{i=[N] \setminus j} \Sigma_i^{-1} \Big)^{-1},\\
    &= \Big( M_1 \frac{1}{\sigma_{p_1}^2}I_d  + M_2 \frac{1}{\sigma_{p_2}^2} I_d + \mydots + M_l \frac{1}{\sigma_{p_l}^2}I_d \Big)^{-1},\\
    &= \frac{\prod_{i\in [l]} \sigma_{p_i}^2}{\sum_{i\in [l]} M_i \prod_{k \in [l] \setminus i} \sigma_{p_k}^2} I_d,
\end{align}
where
\begin{align}
    M_i= \left\{
    \begin{array}{@{}ll@{}}
    N_i & \text{if}\ c_j \notin \mathcal{C}_i \\
    N_i-1, & \text{if}\ c_j \in \mathcal{C}_i
    \end{array}\right. .
\end{align}

Therefore, we have the following

\begin{align}
    \boldsymbol{\phi}^{\setminus j} = \frac{1}{\sum_{i\in [l]} M_i \prod_{k \in [l] \setminus i} \sigma_{p_k}^2} \Big( \sum_{i=[l]} \prod_{k_1 \in [l] \setminus i} \sigma_{p_{k_1}}^2 \sum_{\substack{c_{k_2} \in \mathcal{C}_i\\ k_2 \neq j}} \boldsymbol{\psi}_i \Big)+ \bp_{\boldsymbol{\phi}^{\setminus j}}.
\end{align}
Then, the client uses $\Sigma_{\boldsymbol{\phi}^{\setminus j}}$ and $\hat{\boldsymbol{\phi}}_j$ to estimate $\boldsymbol{\theta}_{j}^*$ as
\begin{align}
    \boldsymbol{\theta}_{j}^* &= \Sigma_{\boldsymbol{\theta}_{j}^*} \Big( \big((\frac{\prod_{i\in [l]} \sigma_{p_i}^2}{\sum_{i\in [l]} M_i \prod_{k \in [l] \setminus i} \sigma_{p_k}^2}+\tau^2) I_d \big)^{-1} \boldsymbol{\phi}^{\setminus j} + (\sigma_c^2-\tau^2)^{-1} \hat{\boldsymbol{\phi}}_j \Big)+\bp_{\boldsymbol{\theta}_{j}^*} ,\label{sigma_j_estimate2}
\end{align}
where $\bp_{\boldsymbol{\theta}_j^*} \sim \mathcal{N}(\boldzero, \Sigma_{\boldsymbol{\theta}_{j}^*})$ and
\begin{align}
    \Sigma_{\boldsymbol{\theta}_{j}^*} &= \bigg( \Big( \frac{\prod_{i\in [l]} \sigma_{p_i}^2 + \tau^2 \sum_{i\in [l]} M_i \prod_{k \in [l] \setminus i} \sigma_{p_k}^2 }{\sum_{i\in [l]} M_i \prod_{k \in [l] \setminus i} \sigma_{p_k}^2}  \Big)^{-1} + (\sigma_c^2-\tau^2)^{-1} \bigg)^{-1} I_d.
\end{align}
We expand \eqref{sigma_j_estimate2} as
\begin{align}
    \boldsymbol{\theta}_{j}^* = \Sigma_{\boldsymbol{\theta}_{j}^*} (\sigma_c^2-\tau^2)^{-1} \hat{\boldsymbol{\phi}}_j +  \Sigma_{\boldsymbol{\theta}_{j}^*} \Big( \frac{1}{\prod_{i\in [l]} \sigma_{p_i}^2 + \tau^2 \sum_{i\in [l]} M_i \prod_{k \in [l] \setminus i} \sigma_{p_k}^2 } \sum_{i=[l]} \prod_{k_1 \in [l] \setminus i} \sigma_{p_{k_1}}^2 \sum_{\substack{c_{k_2} \in \mathcal{C}_i\\ k_2 \neq j}} \boldsymbol{\psi}_i \Big)+\bp_{\boldsymbol{\theta}_{j}^*}. \label{minimizer_theta_l_privacy}
\end{align}
This is the Bayes optimal solution to the local Bayes objective optimization problem for client $c_j$ in \eqref{local_regression}. Next, we know in \textsc{FedHDP} the clients do not have access to individual client updates, but rather the global model. As a result, the clients solve the \textsc{FedHDP} local objective in \eqref{local_objective_opt}. Given a value of $\lambda_j$ and the global estimate $\hat{\boldsymbol{\theta}}^*$, the minimizer $\hat{\boldsymbol{\theta}}_j(\lambda_j)$ of \eqref{local_objective_opt} is
\begin{align}
    \hat{\boldsymbol{\theta}}_j(\lambda_j) &= \frac{1}{1+\lambda_j} \Big( \hat{\boldsymbol{\phi}}_j+\lambda_j \hat{\boldsymbol{\theta}}^* \Big)\\
    &=\frac{1}{1+ \lambda_j} \bigg( 
    \frac{\sum_{i=[l]} N_i r_i+\lambda_j i_j}{ \sum_{i=[l]} N_i r_i}\hat{\boldsymbol{\phi}}_j + \frac{1}{\sum_{i=[l]} N_i r_i} \sum_{i \in [l]} \lambda_j r_i \sum_{\substack{c_k\in \mathcal{C}_i \\k \neq j}} \boldsymbol{\psi}_k \bigg), \label{minimizer_lambda_l_privacy}
\end{align}
where $i_j=r_i$ for client $c_j \in \mathcal{C}_i$. Note that we have $l+1$ terms in both \eqref{minimizer_lambda_l_privacy} and \eqref{minimizer_theta_l_privacy}, which we can use to compute the value of $\lambda_j^*$ as done in a previous part, and results similar to the ones in the prior parts of this appendix then follow from such findings. Note that computing the expressions of the optimal $\lambda_j^*$ in closed form for each one of the $l$ sets of private clients in the considered setup is involved; hence, our brief presentation of the sketch of the solution.

\section{Federated Point Estimation}\label{fpe_appendix}
In this appendix, we provide a brief discussion of a special case of the considered federated linear regression and make use of the results stated in Appendix \ref{LR_appendix}. In the federated point estimation problem, $\bF_j=[1,1,\mydots,1]^T$ of length $n_s$ is available at client $c_j$. For reference, the federated point estimation algorithm is described in Algorithm \ref{FedHDP_FPE_algorithm}. The results in the previous appendix can be used for federated point estimation by setting $d=1$. In the remainder of this appendix, we assume the opt-out of privacy scenario where clients choose to be either private or non-private in the setup. First, we state the global estimate optimality in Lemma \ref{ratio_lemma} and show its proof.

\begin{algorithm*}[h!]
    \caption{\textsc{FedHDP:} Federated Learning with Heterogeneous Differential Privacy (Point Estimation)}
    \begin{multicols}{2}
    \begin{algorithmic}\label{FedHDP_FPE_algorithm}
        \STATE \textit{Inputs:} $\theta^0$, $\alpha^2$, $\tau^2$, $\gamma^2$, $\eta=1$, $\{\lambda_j\}_{j \in [N]}$, $r$, $\rho_{\text{np}}$, $N$.
        \STATE \textit{Outputs:} $\theta^*, \{\theta_j^*\}_{j \in [N]}$
        \STATE \textbf{At server:}
        \FOR{client $c_j$ in $\mathcal{C}^t$ \textbf{in parallel}}
        \STATE $\psi_j \leftarrow$ \textit{ClientUpdate}($\theta^t, c_j$)
        \ENDFOR
        \STATE $\theta^* \leftarrow \frac{1}{\rho_{\text{np}} N+r(1-\rho_{\text{np}}) N} \sum_{c_i \in \mathcal{C}_{\text{np}}} \psi_i$
        \STATE \hspace{0.33in} $+\frac{r}{\rho_{\text{np}} N+r(1-\rho_{\text{np}}) N} \sum_{c_i \in \mathcal{C}_{\text{p}}} \psi_i$
        \columnbreak
        \STATE \textbf{At client $c_j$:}
        \STATE \textit{ClientUpdate}($\theta^0, c_j$):
        \STATE $\theta \leftarrow \theta^0$
        \STATE $\theta_j \leftarrow \theta^0$
        \STATE $\theta \leftarrow \theta - \eta (\theta-\frac{1}{n_s}\sum_{i=1}^{n_s} x_{j,i} )$
        \STATE $ \theta_j^* \leftarrow \theta_j - \eta_j \big( (\theta_j-\frac{1}{n_s}\sum_{i=1}^{n_s} x_{j,i})+\lambda_j (\theta_j-\theta) \big)$
        \STATE $\psi \leftarrow \theta + \mathbbm{1}_{c_j \in \mathcal{C}_{\text{p}}} \mathcal{N}(0,(1-\rho_{\text{np}}) N\gamma^2)$
        \STATE return $\psi$ to server
    \end{algorithmic}
    \end{multicols}
\end{algorithm*}

\begin{lemma}[Global estimate optimality]\label{ratio_lemma}
\textsc{FedHDP} from the server's point of view, with ratio $r^*$ chosen below, is Bayes optimal (i.e., $\theta$ converges to $\theta^*$) in the considered federated point estimation problem given by
    $r^*=\frac{\sigma_c^2}{\sigma_c^2 + N_{\text{p}} \gamma^2}.$
Furthermore, the resulting variance is:
\begin{align}
    \sigma_{s,\emph{opt}}^{2} = \frac{1}{N} \left[ \frac{\sigma_c^2 (\sigma_c^2 + N_{\text{p}} \gamma^2) }{\sigma_c^2 + \rho_{\text{np}}N_{\text{p}} \gamma^2} \right].
\end{align}
\end{lemma}
\begin{proof}
Follows directly by setting $d=1$, $\gamma_1^2=0$, $\gamma_2^2=\gamma^2$, $N_1=N_{\text{np}}$, and $N_2=N_{\text{p}}$ in Lemma \ref{opt_global_estimate}.
\end{proof}

Next, we show some simulation results for the server noise $\sigma_s^2$ against the ratio $r$ for different values of $\sigma_c^2$ and $\gamma^2$ in the federated point estimation setup with $N$ clients and $\rho_{\text{np}}$ fraction of non-private clients. 
The results are shown in Figure \ref{sigma_vs_tau_gamma}, and we can see that the optimal ratio $r^*$ in Lemma \ref{ratio_lemma} minimizes the server variance as expected. Additionally, we show the resulting server noise $\sigma_s^2$ versus the fraction of non-private clients $\rho_{\text{np}}$ plotted for two scenarios. The first is the baseline \textsc{FedAvg}, and the second is the optimal \textsc{FedHDP}. We can see in Figure \ref{sigma_vs_rho} that \textsc{FedHDP} provides better noise variance at the server compared to \textsc{FedAvg}, and the gain can be significant for some values of $\rho_{\text{np}}$, even if a small percentage of clients opt out.

\begin{figure}
\captionsetup{font=small}
    \centering
    \includegraphics[width=0.45\textwidth]{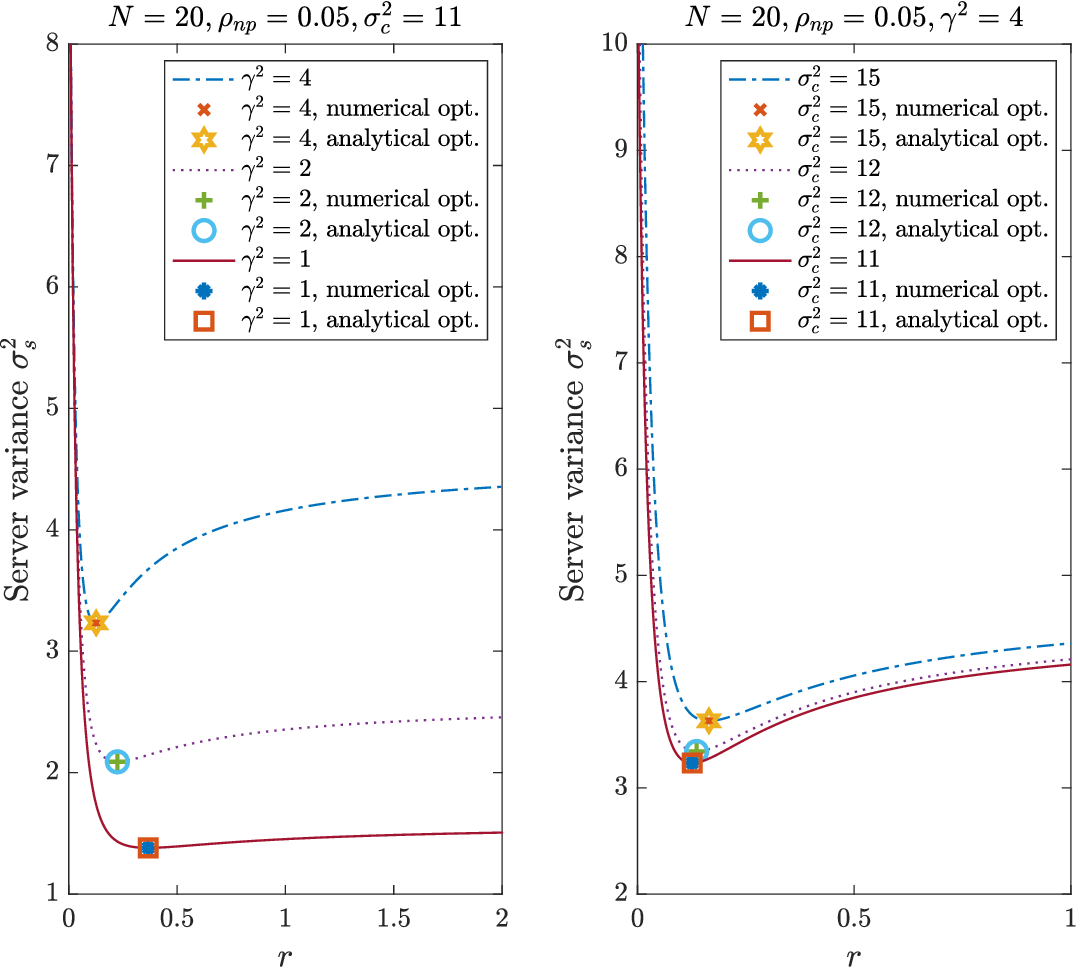}
	\captionof{figure}{Server noise variance $\sigma_s^2$ vs the ratio hyperparameter $r$. (left) Trade-off for different $\gamma^2$, (right) trade-off for different $\sigma_c^2=\alpha^2+\tau^2$.}
	\label{sigma_vs_tau_gamma}
\end{figure}

Furthermore, we consider the performance gap between the optimal solution and the other baselines in this setup. The following lemma states such results.

\begin{lemma}[Global model performance gap]
The server-side estimation mean-square error gap between \textsc{FedHDP} (optimal) and the baselines, \textsc{HDP-FedAvg} and \textsc{DP-FedAvg}, is as follows:
\begin{gather}
    \sigma_{s,\text{hdp-fedavg}}^2-\sigma_{s,opt}^2 = \frac{\rho_{\text{np}}(1-\rho_{\text{np}})^3  \gamma^4 N}{\alpha^2+\tau^2+\rho_{\text{np}}(1-\rho_{\text{np}}) \gamma^2 N} \geq 0,\label{server_noise_fedavg}\\
    \sigma_{s,\text{dp-fedavg}}^2 - \sigma_{s,\text{opt}}^2 = \frac{\rho_{\text{np}}(1-\rho_{\text{np}}) (\alpha^2+\tau^2) \gamma^2 + \rho_{\text{np}}(1-\rho_{\text{np}})^2 \gamma^4 N}{ \alpha^2+\tau^2+ \rho_{\text{np}}(1-\rho_{\text{np}}) \gamma^2 N} \geq 0,\label{server_noise_dpfedavg}\\
    \sigma_{s,\text{dp-fedavg}}^2 -  \sigma_{s,\text{hdp-fedavg}}^2 =  
    \frac{\rho_{\text{np}}(1-\rho_{\text{np}}) (\alpha^2+\tau^2) \gamma^2+\rho_{\text{np}}^2(1-\rho_{\text{np}})^2\gamma^4 N}{ \alpha^2+\tau^2+ \rho_{\text{np}}(1-\rho_{\text{np}}) \gamma^2 N} \geq 0.
\end{gather}
\end{lemma}

Notice that if $\rho_{\text{np}} \to 0$ (homogeneous private clients) or $\rho_{\text{np}} \to 1$ (no private clients), the gap vanishes as expected. In other words, the benefit of \textsc{FedHDP} on the server side is only applicable in the case of heterogeneous differential privacy.
It can be observed that if the number of clients is large ($N \to \infty$), the gap approaches $(1-\rho_{\text{np}})^2\gamma^2$ and $(1-\rho_{\text{np}})\gamma^2$ in \eqref{server_noise_fedavg} and \eqref{server_noise_dpfedavg}, respectively. Notice that having this constant gap is in contrast to $\sigma_{s,opt}^2$ vanishing as $N\to \infty$. This is expected since the noise in the observation itself decreases as the number of clients increases and, hence, having the non-private clients alone would be sufficient to (perfectly) learn the optimal global model.  Finally, if the noise $\gamma^2$ added for the privacy is large $(\gamma^2 \to \infty)$, which corresponds to a small $\epsilon,$ then the gap with optimality grows unbounded. In contrast, in this case, $\sigma_{s,opt}^2$ remains bounded, again because the optimal aggregation strategy would be to discard the private clients and to only aggregate the non-private updates. 

Follows the local estimate optimality lemma and its proof.
\begin{theorem}[Local estimate optimality]
Assuming using \textsc{FedHDP} with ratio $r^*$ in Lemma \ref{ratio_lemma}, and using the values $\lambda_{\text{np}}^*$ for non-private clients and $\lambda_{\text{p}}^*$ for private clients stated below, \textsc{FedHDP} is Bayes optimal (i.e., $\theta_j$ converges to $\theta^*_j$ for each client $j \in [N]$)
\begin{gather}
    \lambda_{\text{np}}^* = \frac{1}{\Upsilon^2}, \\
    \lambda_{\text{p}}^* =\frac{N + N \Upsilon^2 + (N-N_{\text{p}}) \Gamma^2} {N\Upsilon^2(\Upsilon^2+1) + (N-N_{\text{p}}+1) \Upsilon^2\Gamma^2+\Gamma^2}.
\end{gather}
where $\Upsilon^2 = \frac{\tau^2}{\alpha^2}$ and $\Gamma^2 = \frac{N_\text{p} \gamma^2}{\alpha^2}$.
\end{theorem}
\begin{proof}
Follows directly by setting $d=1$, $\Gamma_1^2=0$, $\Gamma_2^2=\Gamma^2$, $N_1=N_{\text{np}}$, and $N_2=N_{\text{p}}$ in Lemma \ref{opt_personalized_estimate}.
\end{proof}

We show an additional simulation result for the federated point estimation problem. We compare the effect of opting out by a client in the federated point estimation where we observe some similar effects as in the federated linear regression, see Figure \ref{gain_opting_out_fpe}.

\begin{figure}[t]
\captionsetup{font=small}
	\centering
	\begin{minipage}{.45\textwidth}
		\centering
		\includegraphics[height=2.3in]{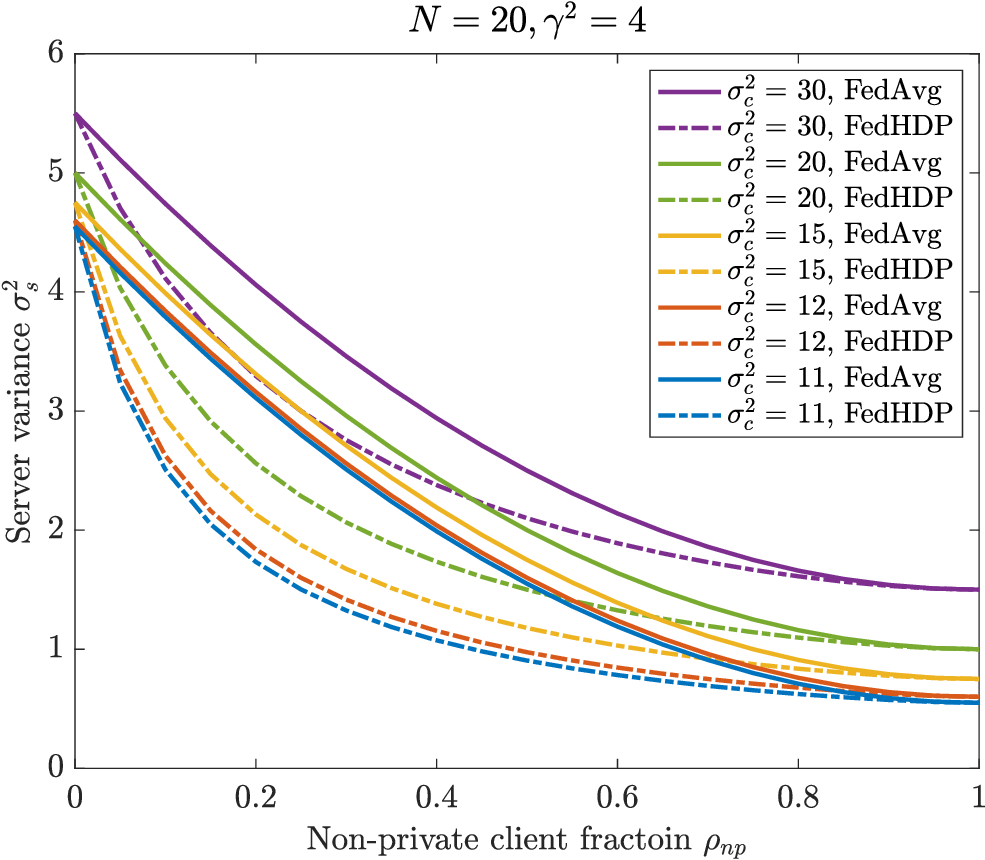}
        \caption{Server noise variance $\sigma_s^2$ vs non-private client fraction $\rho_{\text{np}}$ for the baseline \textsc{FedAvg} aggregator and optimal \textsc{FedHDP} aggregator.}
        \label{sigma_vs_rho}
	\end{minipage}
	\hspace{0.03\textwidth}
	\begin{minipage}{.45\textwidth}
	\begin{center}
        \includegraphics[height=2.3in]{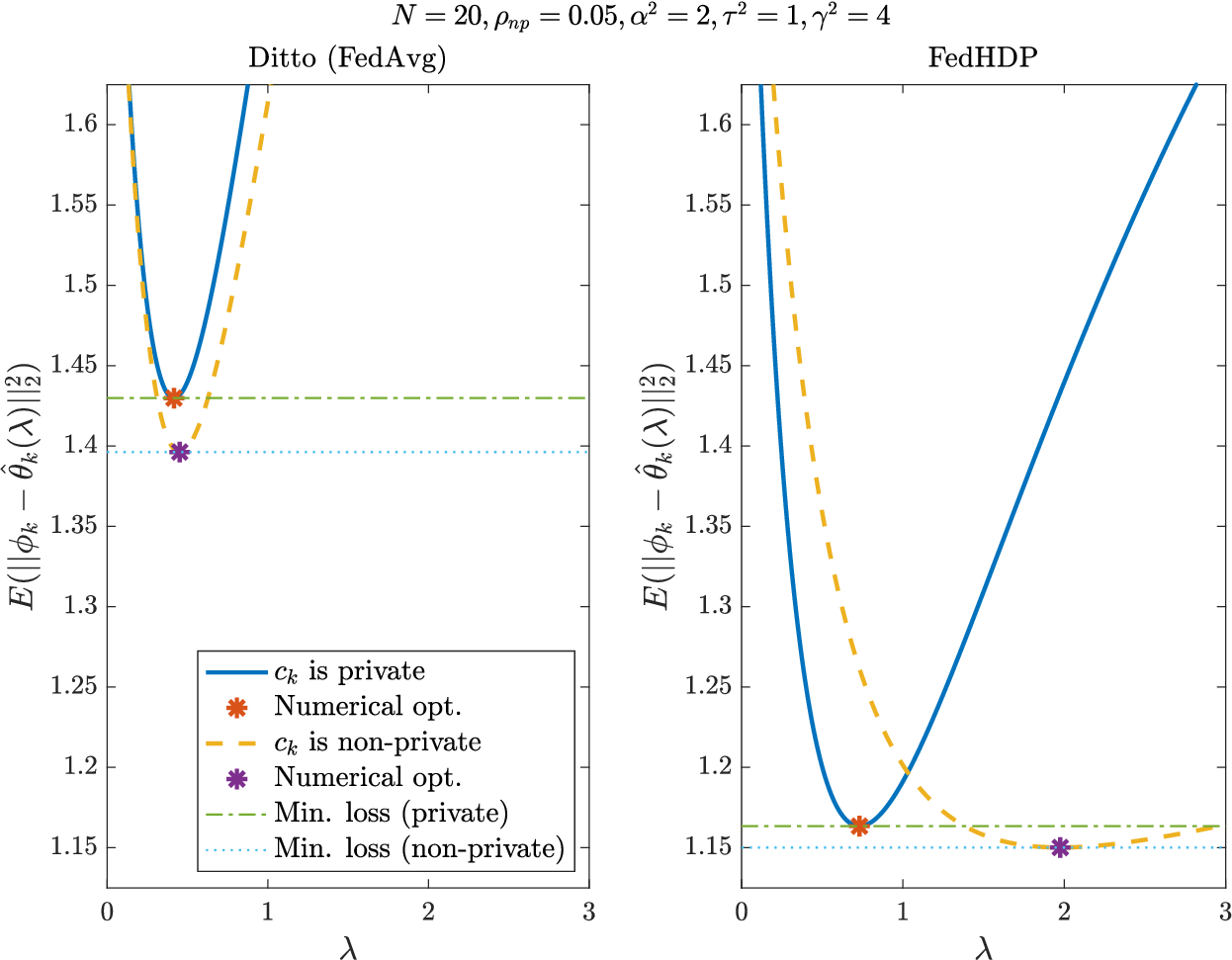}
        \caption{The effect of opting out on the personalized local model estimate as a function of $\lambda$ when employing (left) \textsc{Ditto} with vanilla \textsc{FedAvg} and (right) \textsc{FedHDP}.}
        \label{gain_opting_out_fpe}
	\end{center}
	\end{minipage}
\end{figure}

It is worth mentioning that we notice that the values of $\lambda^*$ are different for private and non-private clients. We recall that the derived expression for the personalization parameters for all clients considers the presence of data heterogeneity as well as privacy heterogeneity. In Table \ref{special_cases_lambda_fpe} we provide a few important special cases for both $\lambda_{\text{p}}^*$ and $\lambda_{\text{np}}^*$ for the considered federated point estimation problem.

\begin{table}[ht!]
	\caption{Special cases of \textsc{FedHDP} in the federated point estimation.}
	\vspace{-.15in}
	\label{special_cases_lambda_fpe}
	\begin{center}
        \resizebox{\linewidth}{!}{
		\footnotesize
		\begin{tabular}{|l|ll|ll|}
            \hline
            \rowcolor{Gray} & \multicolumn{2}{c|}{\textbf{No privacy ($\rho_{\text{np}} = 1$)}} & \multicolumn{2}{c|}{\textbf{Homogeneous privacy ($\rho_{\text{np}} = 0$)}} \\ \hline
			\cellcolor{Gray} \textbf{Homogeneous data ($\tau^2 = 0$)}   & \textsc{FedAvg}~~~ &  $\lambda_{\text{np}}^* = \infty$  & \textsc{DP-FedAvg+Ditto}~~~ & $\lambda_{\text{p}}^* = \frac{1}{\gamma^2}$\\ \hline
			\cellcolor{Gray} \textbf{Heterogeneous data ($\tau^2 > 0$)} &  \textsc{FedAvg+Ditto}~~~ & $\lambda_{\text{np}}^* = \frac{\alpha^2}{\tau^2}$ & \textsc{DP-FedAvg+Ditto}~~~ & $\lambda_{\text{p}}^* =  \frac{N \alpha^2}{N \tau^2 + (1-\rho_{\text{np}}) N \gamma^2}$\\ \hline
		\end{tabular}
        }
	\end{center}
\end{table}

{\bf Homogeneous data \& No privacy:} In this case, the optimal \textsc{FedHDP} algorithm recovers \textsc{FedAvg}~\cite{mcmahan2017communication} with no personalization ($\lambda_{\text{np}}^* = \infty$). This is not surprising as this is exactly the setup for which the vanilla federated averaging was originally introduced.

{\bf Heterogeneous data \& No privacy:} In this case, the optimal \textsc{FedHDP} algorithm recovers Ditto~\cite{li2021ditto}. Again, this is not surprising as this is exactly the setup for which Ditto has been shown to be Bayes optimal.

{\bf Homogeneous data \& Homogeneous privacy:} In this case, the optimal \textsc{FedHDP} algorithm recovers \textsc{DP-FedAvg}~\cite{andrew2019differentially}, however, with additional personalization using \textsc{Ditto}~\cite{li2021ditto}. At first, this might be surprising as there is no data heterogeneity in this case, which is where \textsc{Ditto} would be needed. However, a closer look at this case reveals that the noise added due to differential privacy creates artificial data heterogeneity that needs to be dealt with using \textsc{Ditto}. In fact, as $\epsilon \to 0$, or equivalently, as $\gamma^2 \to \infty,$ for the added noise for privacy, we observe that $\lambda_{\text{p}}^* \to 0$ implying that the local learning becomes optimal. This is expected since, in this case, the data from other (private) clients is, roughly speaking, so noisy that it is best to rely solely on local data.

{\bf Heterogeneous data \& Homogeneous privacy:} In this case, the optimal \textsc{FedHDP} algorithm again recovers \textsc{DP-FedAvg+Ditto}. Similar to the homogeneous data case, with $\gamma^2 \to \infty,$ we observe that $\lambda_{\text{p}}^* \to 0,$ i.e., the local learning becomes optimal.

\section{Experiments: Extended Experimental Results}\label{results_app}
In this section, we provide an extended version of the results of experiments conducted on the considered datasets. We describe the datasets along with the associated tasks in Tables \ref{setup_exp}, the models used in Table \ref{setup_models}, and the hyperparameters used in Table \ref{setup_hyp}.

\begin{table}[ht!]
\captionsetup{font=small}
\caption{Experiments setup: Number of clients is $N$, approximate fraction of clients per round is $q$.}
\label{setup_exp}
\begin{center}
\footnotesize
\begin{tabular}{|c|c|c|c|c|}
\hline
\textbf{Dataset} &$N$ &$q$ &\textbf{Task} &\textbf{Model} \\ \hline
non-IID MNIST &$2,\!000$ &$5\%$ &$10$-label classification & FC NN \\ \hline
FMNIST  &$3,\!383$ &$3\%$ &$10$-label classification &FC NN \\ \hline
FEMNIST  &$3,\!400$ &$3\%$ &$62$-label classification & CNN \\ \hline
\end{tabular}
\end{center}
\end{table}

\begin{table}[h!]
\captionsetup{font=small}
	\caption{Models used for experiments.}
	\label{setup_models}
	\begin{center}
		\footnotesize
		\begin{tabular}{|c|c|c|}
			\hline
			\rowcolor{Gray} \multicolumn{3}{|c|}{non-IID MNIST, and FMNIST Datasets} \\ \hline
			\textbf{Layer} & \textbf{Size}  & \textbf{Activation} \\ \hline
			Input image& $28\times 28$ & -\\ \hline
			Flatten& $784$ & - \\ \hline
			Fully connected& $50$ & ReLU\\ \hline
			Fully connected& $10$ & Softmax\\ \hline
			
			\rowcolor{Gray} \multicolumn{3}{|c|}{FEMNIST Dataset} \\ \hline
			Input image& $28\times 28$ & -\\ \hline
			Convolutional (2D)& $28 \times 28 \times 16$  & ReLU\\ \hline
			Max pooling (2D)& $14\times 14 \times 16$  & -\\ \hline
			Convolutional (2D)& $14 \times 14\times 32$  & ReLU\\ \hline
			Max pooling (2D)& $7 \times 7 \times 32$  & -\\ \hline
			Dropout $(25\%)$& - & -\\ \hline
			Flatten& $1568$ & -\\ \hline
			Fully connected& $128$ & ReLU\\ \hline
			Dropout ($50\%$)& - & -\\ \hline
			Fully connected& $62$ & Softmax\\ \hline
		\end{tabular}
	\end{center}
\end{table}

\begin{table}[h!]
\captionsetup{font=small}
	\caption{Hyperparameters used for each experiment.}
	\label{setup_hyp}
	\begin{center}
		\footnotesize
		\begin{tabular}{|c|c|c|c|c|}
			\hline
			\textbf{Hyperparameter}&\textbf{non-IID MNIST} &\textbf{FMNIST} &\textbf{FEMNIST}\\ \hline
			Batch size & \multicolumn{3}{c|}{20}  \\ \hline
			Epochs & 25 & 50 &25 \\ \hline
			$\eta , \eta_p$& 0.5 & 0.01 & 0.02 \\ \hline
			$\eta , \eta_p$ decaying factor& \multicolumn{2}{c|}{0.9 every 50 rounds} & N/A\\ \hline
			$S^0$& 0.5 & 0.5 &2.0\\ \hline
			$\eta_b$& \multicolumn{3}{c|}{0.2} \\ \hline
			$\kappa$ & \multicolumn{3}{c|}{0.5} \\ \hline
			Effective noise multiplier& 1.5 & 4.0 & 1.0\\ \hline
		\end{tabular}
	\end{center}
\end{table}

For each experiment, we presented the results of each dataset for the two baselines, i.e., \textsc{Non-Private}, \textsc{DP-FedAvg}, and \textsc{HDP-FedAvg}, as well as the proposed \textsc{FedHDP} algorithm along with the best parameters that produce the best results in the main body of the paper. In this appendix, we show the extended version of the experiments. For all experiments, training is stopped after $500$ communication rounds for each experiment. The server's test dataset is the test MNIST dataset in the non-IID MNIST experiments or the collection of the test datasets of all clients in the FMNIST and FEMNIST datasets. Note that the experiment of \textsc{FedHDP} with $r=0$ denotes the case where the server only communicates with non-private clients during training and ignores all private clients. For the baseline algorithms, we note that the personalization scheme that is used on clients is \textsc{Ditto}. We vary the ratio hyperparameter $r$ as well as the regularization hyperparameters $\lambda_{\text{p}}$ and $\lambda_{\text{np}}$ at the clients and observe the results. In the following tables, we list the entirety of the results of all experiments conducted on each dataset for various values of the hyperparameters. For readability, we highlight the rows that contain the best values of performance metrics in the proposed algorithm \textsc{FedHDP}.

\begin{table}[ht!]
\captionsetup{font=small}
\caption{Experiment results on \textit{non-IID MNIST},  $(\epsilon, \delta)=(3.6, 10^{-4})$. The variance of the performance metric across clients is between parenthesis.}

\label{results_nonIID_mnist}
\begin{center}
\resizebox{\linewidth}{!}{
\begin{tabular}{|c|c||c||c|c|c||c|c|c|}
\hline
\rowcolor{Gray} \multicolumn{9}{|c|}{$\lambda_{\text{p}} = \lambda_{\text{np}} = 0.005$} \\ \hline
\multicolumn{2}{|c||}{Setup} & \multicolumn{4}{|c||}{Global model} & \multicolumn{3}{|c|}{Personalized local models} \\ \hline
\textbf{Algorithm} & \textbf{hyperparam.} &$Acc_{g}\%$ &$Acc_{\text{g},\text{p}}\%$ & $Acc_{\text{g},\text{np}}\%$& $\bigtriangleup_{\text{g}}\%$ &$Acc_{\text{l},\text{p}}\%$ & $Acc_{\text{l},\text{np}}\%$& $\bigtriangleup_{\text{l}}\%$ \\ \hline
\textsc{Non-Private+Ditto} & - &$93.8$ & - &$93.75 (0.13)$&- &-&$99.98 (0.001)$ &- \\ \hline
\textsc{DP-FedAvg+Ditto} & - &$88.75$ &$88.64 (0.39)$ & - & - &$99.97 (0.002)$ &- &- \\ \hline
\textsc{HDP-FedAvg+Ditto} & - &$87.71$ &$87.55(0.42)$ &$88.35(0.34)$&$0.8$ &$99.97(0.001)$&$99.93(0.001)$ &$-0.04$ \\ \hline
\textsc{FedHDP} & $r\!=\!0$ & $90.7$ &$90.64(0.68)$ &$91.72(0.5)$&$1.08$ &$90.64(0.68)$&$99.94(0.001)$ &$9.2964$ \\ \hline
\textsc{FedHDP} & $r\!=\!0.001$ &$91.74$&$91.65(0.39)$ &$92.61(0.27)$ &$0.94$ &$99.94(0.001)$&$99.95(0.001)$ &$0.01$ \\ \hline
\rowcolor{LightCyan} \textsc{FedHDP} & $r\!=\!0.01$ &$92.48$ &$92.43(0.30)$ &$93.30(0.21)$&$0.88$ &$99.94(0.001)$&$99.94(0.001)$ &$0$ \\ \hline
\textsc{FedHDP} & $r\!=\!0.025$ &$92.36$ &$92.28(0.27)$ &$92.96(0.19)$&$0.68$ &$99.95(0.001)$&$99.91(0.001)$ &$-0.04$ \\ \hline
\textsc{FedHDP} & $r\!=\!0.1$ &$90.7$ &$90.59(0.34)$ &$91.31(0.26)$&$0.73$ &$99.97(0.001)$&$99.95(0.001)$ &$-0.02$ \\ \hline

\rowcolor{Gray} \multicolumn{9}{|c|}{$\lambda_{\text{p}} = \lambda_{\text{np}} = 0.05$} \\ \hline
\textsc{Non-Private+Ditto} & - &$93.81$ & - &$93.76 (0.13)$&- &-&$99.93 (0.001)$ &- \\ \hline
\textsc{DP-FedAvg+Ditto} & - &$87.98$ &$87.97 (0.39)$ & - & - &$99.84 (0.002)$ &- &- \\ \hline
\textsc{HDP-FedAvg+Ditto} & - &$89.64$ &$89.50(0.32)$ &$90.55(0.24)$&$1.05$ &$99.83(0.002)$&$99.84(0.002)$ &$0.01$ \\ \hline
\textsc{FedHDP} & $r\!=\!0$ & $91$ &$91.12(0.48)$ &$92.08(0.41)$&$0.96$ &$91.12(0.48)$&$99.76(0.002)$ &$8.65$ \\ \hline
\textsc{FedHDP} & $r\!=\!0.001$ &$92.15$ &$92.10(0.33)$ &$92.88(0.25)$ &$0.78$ &$99.81(0.002)$&$99.78(0.002)$ &$-0.03$ \\ \hline
\textsc{FedHDP} & $r\!=\!0.01$ &$92.45$ &$92.39(0.33)$ &$93.26(0.25)$&$0.87$ &$99.81(0.002)$&$99.78(0.003)$ &$-0.03$ \\ \hline
\textsc{FedHDP} & $r\!=\!0.025$ & $92.14$&$92.09(0.35)$ &$93.01(0.26)$&$0.92$ &$99.85(0.002)$&$99.8(0.002)$ &$-0.05$ \\ \hline
\textsc{FedHDP} & $r\!=\!0.1$ &$90.7$ &$90.82(0.29)$ &$91.55(0.21)$&$0.73$ &$99.87(0.002)$&$99.80(0.003)$ &$-0.06$ \\ \hline

\rowcolor{Gray} \multicolumn{9}{|c|}{$\lambda_{\text{p}} = \lambda_{\text{np}} = 0.25$} \\ \hline
\textsc{Non-Private+Ditto} & - &$93.79$ & - &$93.75 (0.13)$&- &-&$99.10 (0.007)$ &- \\ \hline
\textsc{DP-FedAvg+Ditto} & - &$88.26$ &$88.23 (0.41)$ & - & - &$98.23 (0.017)$ &- &- \\ \hline
\textsc{HDP-FedAvg+Ditto} & - &$88.27$ &$88.09(0.49)$ &$89.08(0.4)$&$0.99$ &$98.14(0.017)$&$98.13(0.017)$ &$-0.01$ \\ \hline
\textsc{FedHDP} & $r\!=\!0$ & $90.42$ &$90.41(0.69)$ &$91.41(0.58)$&$1.0$ &$90.41(0.69)$&$98.06(0.023)$ &$7.65$ \\ \hline
\textsc{FedHDP} & $r\!=\!0.001$ &$92.18$&$92.12(0.34)$ &$92.85(0.26)$ &$0.73$ &$98.47(0.015)$&$98.08(0.024)$ &$-0.39$ \\ \hline
\textsc{FedHDP} & $r\!=\!0.01$ &$92.41$ &$92.35(0.29)$ &$93.19(0.21)$&$0.83$ &$98.62(0.015)$&$98.33(0.019)$ &$-0.28$ \\ \hline
\textsc{FedHDP} & $r\!=\!0.025$ &$92.5$ &$92.42(0.28)$ &$93.19(0.19)$&$0.77$ &$98.71(0.011)$&$98.41(0.017)$ &$-0.3$ \\ \hline
\textsc{FedHDP} & $r\!=\!0.1$ &$91.17$ &$91.10(0.32)$ &$91.94(0.24)$&$0.84$ &$98.71(0.012)$&$98.60(0.013)$ &$-0.11$ \\ \hline

\end{tabular}
}
\end{center}
\end{table}

\begin{table}[ht!]
\captionsetup{font=small}
\caption{Experiment results on \textit{Skewed non-IID MNIST},  $(\epsilon, \delta)=(3.6, 10^{-4})$. The variance of the performance metric across clients is between parenthesis.}
\label{results_skewed_nonIID_mnist}
\begin{center}
\resizebox{\linewidth}{!}{
\begin{tabular}{|c|c||c||c|c|c||c|c|c|}
\hline
\rowcolor{Gray} \multicolumn{9}{|c|}{$\lambda_{\text{p}} = \lambda_{\text{np}} = 0.005$} \\ \hline
\multicolumn{2}{|c||}{Setup} & \multicolumn{4}{|c||}{Global model} & \multicolumn{3}{|c|}{Personalized local models} \\ \hline
\textbf{Algorithm} & \textbf{hyperparam.} &$Acc_{g}\%$ &$Acc_{\text{g},\text{p}}\%$ & $Acc_{\text{g},\text{np}}\%$& $\bigtriangleup_{\text{g}}\%$ &$Acc_{\text{l},\text{p}}\%$ & $Acc_{\text{l},\text{np}}\%$& $\bigtriangleup_{\text{l}}\%$ \\ \hline
\textsc{Non-Private+Ditto} & - &$93.67$ & - &$93.62 (0.15)$&- &-&$99.98 (0.001)$ &- \\ \hline
\textsc{DP-FedAvg+Ditto} & - &$88.93$ &$88.87 (0.35)$ & - & - &$99.98 (0.001)$ &- &- \\ \hline
\textsc{HDP-FedAvg+Ditto} & - &$88.25$ &$88.05(0.39)$ &$89.98(0.05)$&$1.93$ &$99.97(0.001)$&$99.85(0.001)$ &$-0.11$ \\ \hline
\textsc{FedHDP} & $r\!=\!0$ & $10.27$ &$7.1(6.5)$ &$100(0)$&$92.9$ &$7.1(6.5)$&$100(0)$ &$92.9$ \\ \hline
\textsc{FedHDP} & $r\!=\!0.025$ &$87.11$&$86.61(1.10)$ &$98.16(0.01)$ &$11.55$ &$99.99(0.001)$&$99.91(0.001)$ &$-0.08$ \\ \hline
\rowcolor{LightCyan} \textsc{FedHDP} & $r\!=\!0.1$ &$90.36$ &$89.96(0.37)$ &$97.45(0.01)$&$7.49$ &$99.97(0.001)$&$99.76(0.003)$ &$-0.21$ \\ \hline
\textsc{FedHDP} & $r\!=\!0.5$ &$88.44$ &$88.14(0.36)$ &$93.36(0.03)$&$5.2$ &$99.98(0.001)$&$99.93(0.001)$ &$-0.05$ \\ \hline
\textsc{FedHDP} & $r\!=\!0.75$ &$89.14$ &$88.92(0.37)$ &$92.43(0.06)$&$3.5$ &$99.97(0.001)$&$99.93(0.001)$ &$-0.04$ \\ \hline
\rowcolor{LightCyan} \textsc{FedHDP} & $r\!=\!0.9$ &$87.96$ &$87.69(0.56)$ &$92.97(0.04)$&$5.28$ &$99.98(0.001)$&$99.96(0.001)$ &$-0.02$ \\ \hline

\rowcolor{Gray} \multicolumn{9}{|c|}{$\lambda_{\text{p}} = \lambda_{\text{np}} = 0.05$} \\ \hline
\textsc{Non-Private+Ditto} & - &$93.67$ & - &$93.62 (0.15)$&- &-&$99.93 (0.001)$ &- \\ \hline
\textsc{DP-FedAvg+Ditto} & - &$88.78$ &$88.70 (0.53)$ & - & - &$99.83 (0.002)$ &- &- \\ \hline
\textsc{HDP-FedAvg+Ditto} & - &$88.33$ &$88.11(0.46)$ &$91.67(0.04)$&$3.56$ &$99.87(0.001)$&$99.61(0.001)$ &$-0.26$ \\ \hline
\textsc{FedHDP} & $r\!=\!0$ & $10.28$ &$7.1(6.5)$ &$100(0)$&$92.9$ &$7.1(6.5)$&$100(0)$ &$92.9$ \\ \hline
\textsc{FedHDP} & $r\!=\!0.025$ &$87.92$&$87.45(0.99)$ &$98.1(0.01)$ &$10.65$ &$99.95(0.001)$&$99.75(0.003)$ &$-0.2$ \\ \hline
\textsc{FedHDP} & $r\!=\!0.1$ &$88.98$ &$88.64(0.52)$ &$96.18(0.02)$&$7.54$ &$99.9(0.001)$&$99.47(0.005)$ &$-0.43$ \\ \hline
\textsc{FedHDP} & $r\!=\!0.5$ &$88.22$ &$87.9(0.38)$ &$93.43(0.03)$&$5.33$ &$99.85(0.002)$&$99.42(0.008)$ &$-0.42$ \\ \hline
\textsc{FedHDP} & $r\!=\!0.75$ &$88.56$ &$88.37(0.35)$ &$91.33(0.04)$&$2.94$ &$99.84(0.002)$&$99.52(0.004)$ &$-0.33$ \\ \hline
\textsc{FedHDP} & $r\!=\!0.9$ &$89.19$ &$88.97(0.4)$ &$92.24(0.03)$&$3.27$ &$99.88(0.001)$&$99.58(0.005)$ &$-0.3$ \\ \hline

\rowcolor{Gray} \multicolumn{9}{|c|}{$\lambda_{\text{p}} = \lambda_{\text{np}} = 0.25$} \\ \hline
\textsc{Non-Private+Ditto} & - &$93.67$ & - &$93.62 (0.15)$&- &-&$99.09 (0.007)$ &- \\ \hline
\textsc{DP-FedAvg+Ditto} & - &$87.78$ &$87.71 (0.53)$ & - & - &$98.15 (0.02)$ &- &- \\ \hline
\textsc{HDP-FedAvg+Ditto} & - &$89.4$ &$89.22(0.26)$ &$92.01(0.03)$&$2.79$ &$98.27(0.02)$&$97.62(0.03)$ &$-0.64$ \\ \hline
\textsc{FedHDP} & $r\!=\!0$ & $10.27$ &$7.1(6.5)$ &$100(0)$&$92.9$ &$7.1(6.5)$&$100(0)$ &$92.9$ \\ \hline
\textsc{FedHDP} & $r\!=\!0.025$ &$87.51$&$87.01(0.9)$ &$98.49(0.01)$ &$11.48$ &$98.69(0.01)$&$99.09(0.006)$ &$-0.4$ \\ \hline
\textsc{FedHDP} & $r\!=\!0.1$ &$89.05$ &$88.66(0.54)$ &$96.8(0.02)$&$8.14$ &$98.69(0.012)$&$98.55(0.008)$ &$-0.13$ \\ \hline
\textsc{FedHDP} & $r\!=\!0.5$ &$88.18$ &$88.11(0.55)$ &$93.43(0.03)$&$5.32$ &$98.32(0.014)$&$97.80(0.01)$ &$-0.52$ \\ \hline
\textsc{FedHDP} & $r\!=\!0.75$ &$87.96$ &$87.8(0.33)$ &$92.58(0.03)$&$4.78$ &$98.25(0.017)$&$97.5(0.02)$ &$-0.75$ \\ \hline
\textsc{FedHDP} & $r\!=\!0.9$ &$88.26$ &$87.93(0.41)$ &$91.67(0.03)$&$3.74$ &$98.25(0.02)$&$97.68(0.02)$ &$-0.57$ \\ \hline
\end{tabular}
}
\end{center}
\end{table}

\begin{table}[ht!]
\captionsetup{font=small}
\caption{Experiment results on \textit{FMNIST},  $(\epsilon, \delta)=(0.6, 10^{-4})$. The variance of the performance metric across clients is between parenthesis.}
\label{results_Fmnist}
\begin{center}
\resizebox{\linewidth}{!}{
\begin{tabular}{|c|c||c||c|c|c||c|c|c|}
\hline
\rowcolor{Gray} \multicolumn{9}{|c|}{$\lambda_{\text{p}} = \lambda_{\text{np}} = 0.005$} \\ \hline
\multicolumn{2}{|c||}{Setup} & \multicolumn{4}{|c||}{Global model} & \multicolumn{3}{|c|}{Personalized local models} \\ \hline
\textbf{Algorithm} & \textbf{hyperparam.} &$Acc_{g}\%$ &$Acc_{\text{g},\text{p}}\%$ & $Acc_{\text{g},\text{np}}\%$& $\bigtriangleup_{\text{g}}\%$ &$Acc_{\text{l},\text{p}}\%$ & $Acc_{\text{l},\text{np}}\%$& $\bigtriangleup_{\text{l}}\%$ \\ \hline
\textsc{Non-Private+Ditto} & - &$89.65$ & - &$89.35 (1.68)$&- &-&$93.95 (0.67)$ &- \\ \hline
\textsc{DP-FedAvg+Ditto} & - &$71.76$ &$71.42 (2.79)$ & - & - &$91.01 (0.94)$ &- &- \\ \hline
\textsc{HDP-FedAvg+Ditto} & - &$75.87$ &$75.77(2.84)$ &$74.41(2.8)$&$-1.36$ &$90.45(1.02)$&$92.32(0.8)$ &$1.87$ \\ \hline
\textsc{FedHDP} & $r\!=\!0$ & $81.78$ &$80.73(2.45)$ &$89.35(1.5)$&$8.62$ &$80.73(2.4)$&$95.80(0.39)$ &$15.06$ \\ \hline
\textsc{FedHDP} & $r\!=\!0.01$ &$85.38$&$84.61(2.05)$ &$89.3(1.26)$ &$4.69$ &$93.26(0.74)$&\cellcolor{LightCyan} $95.94(0.41)$ &$2.67$ \\ \hline
\textsc{FedHDP} & $r\!=\!0.025$ &$85.7$&$84.93(1.97)$ &$89.58(1.29)$ &$4.65$ &$93.04(0.76)$&$95.22(0.54)$ &$2.18$ \\ \hline
\textsc{FedHDP} & $r\!=\!0.05$ &$85.21$ &$84.68(1.99)$ &$86.22(1.76)$&$1.54$ &$92.87(0.74)$&$95.40(0.51)$ &$2.53$ \\ \hline
\textsc{FedHDP} & $r\!=\!0.1$ &$81.76$ &$81.45(2.45)$ &$81.96(1.84)$&$0.51$ &$92.47(0.78)$&$94.83(0.52)$ &$2.36$ \\ \hline
\textsc{FedHDP} & $r\!=\!0.5$ &$78.19$ &$78.02(2.59)$ &$76.48(3.02)$&$-1.53$ &$91.08(0.94)$&$92.59(0.83)$ &$1.51$ \\ \hline

\rowcolor{Gray} \multicolumn{9}{|c|}{$\lambda_{\text{p}} = \lambda_{\text{np}} = 0.05$} \\ \hline
\textsc{Non-Private+Ditto} & - &$89.65$ & - &$89.35 (1.68)$&- &-&$94.53 (0.59)$ &- \\ \hline
\textsc{DP-FedAvg+Ditto} & - &$77.61$ &$77.62(2.55)$ & - & - &$90.04 (1.04)$ &- &- \\ \hline
\textsc{HDP-FedAvg+Ditto} & - &$72.42$ &$77.14(2.72)$ &$76.28(2.76)$&$-0.86$ &$89.12(1.15)$&$90.92(0.91)$ &$1.8$ \\ \hline
\textsc{FedHDP} & $r\!=\!0$ &$82.61$&$80.72(2.45)$ &$89.45(1.51)$ &$8.73$ &$80.72(2.45)$&$95.57(0.38)$ &$14.84$ \\ \hline
\rowcolor{LightCyan} \textsc{FedHDP} & $r\!=\!0.01$ &$86.88$&$85.36(1.89)$ &$90.02(1.28)$ &$4.66$ &$93.76(0.68)$ &\cellcolor{white}$95.78(0.36)$ &\cellcolor{white} $2.02$ \\ \hline
\textsc{FedHDP} & $r\!=\!0.025$ &$86.03$&$84.22(1.98)$ &$88.40(1.68)$ &$4.18$ &$93.53(0.68)$&$95.11(0.54)$ &$0.52$ \\ \hline
\textsc{FedHDP} & $r\!=\!0.05$ &$84.65$ &$82.68(2.16)$ &$86.68(1.67)$&$4.00$ &$92.92(0.76)$&$95.02(0.55)$ &$2.1$ \\ \hline
\textsc{FedHDP} & $r\!=\!0.1$ &$82.89$ &$81.72(2.28)$ &$83.68(2.18)$&$1.96$ &$92.38(0.83)$&$94.25(0.61)$ &$1.87$ \\ \hline
\textsc{FedHDP} & $r\!=\!0.5$ &$76.59$ &$78.05(2.60)$ &$78.04(2.66)$&$-0.0041$ &$89.63(1.10)$&$91.67(0.84)$ &$2.04$ \\ \hline

\rowcolor{Gray} \multicolumn{9}{|c|}{$\lambda_{\text{p}} = \lambda_{\text{np}} = 0.25$} \\ \hline
\textsc{Non-Private+Ditto} & - &$89.66$ & - &$89.36 (1.69)$&- &-&$94.32 (0.64)$ &- \\ \hline
\textsc{DP-FedAvg+Ditto} & - &$70.1$ &$70.40 (2.91)$ & - & - &$88.38 (1.25)$ &- &- \\ \hline
\textsc{HDP-FedAvg+Ditto} & - &$72.67$ &$72.05(2.83)$ &$74.31(2.42)$&$2.26$ &$87.54(1.34)$&$87.39(1.23)$ &$-0.15$ \\ \hline
\textsc{FedHDP} & $r\!=\!0$ & $81.93$ &$80.85(2.39)$ &$89.71(1.39)$&$8.86$ &$80.85(2.39)$&$94.56(0.50)$ &$13.71$ \\ \hline
\textsc{FedHDP} & $r\!=\!0.01$ &$85.31$&$84.55(1.98)$ &$89.27(1.54)$ &$4.72$ &$92.76(0.78)$&$94.77(0.5)$ &$2.01$ \\ \hline
\textsc{FedHDP} & $r\!=\!0.025$ &$86.17$&$85.52(1.92)$ &$89.25(1.31)$ &$3.73$ &$92.46(0.85)$&$94.35(0.57)$ &$1.89$ \\ \hline
\textsc{FedHDP} & $r\!=\!0.05$ &$83.97$ &$83.5(2.19)$ &$85.4(1.88)$&$1.9$ &$91.69(0.91)$&$93.9(0.53)$ &$2.21$ \\ \hline
\textsc{FedHDP} & $r\!=\!0.1$ &$83.78$ &$83.22(2.11)$ &$84.94(2.12)$&$1.72$ &$90.9(1.02)$&$92.62(0.73)$ &$1.72$ \\ \hline
\textsc{FedHDP} & $r\!=\!0.5$ &$74.64$ &$74.63(3.23)$ &$72.54(2.93)$&$-2.09$ &$88.12(1.34)$&$88.69(1.28)$ &$0.57$ \\ \hline

\end{tabular}
}
\end{center}
\end{table}

\begin{table}[ht!]
\captionsetup{font=small}
\caption{Experiment results on \textit{FEMNIST},  $(\epsilon, \delta)=(4.1, 10^{-4})$. The variance of the performance metric across clients is between parenthesis.}
\label{results_Femnist}
\begin{center}
\resizebox{\linewidth}{!}{
\begin{tabular}{|c|c||c||c|c|c||c|c|c|}
\hline
\rowcolor{Gray} \multicolumn{9}{|c|}{$\lambda_{\text{p}} = \lambda_{\text{np}} = 0.005$} \\ \hline
\multicolumn{2}{|c||}{Setup} & \multicolumn{4}{|c||}{Global model} & \multicolumn{3}{|c|}{Personalized local models} \\ \hline
\textbf{Algorithm} & \textbf{hyperparam.} &$Acc_{g}\%$ &$Acc_{\text{g},\text{p}}\%$ & $Acc_{\text{g},\text{np}}\%$& $\bigtriangleup_{\text{g}}\%$ &$Acc_{\text{l},\text{p}}\%$ & $Acc_{\text{l},\text{np}}\%$& $\bigtriangleup_{\text{l}}\%$ \\ \hline
\textsc{Non-Private+Ditto} & - &$81.56$ & - &$81.72 (1.37)$&- &-&$73.86 (1.5)$ &- \\ \hline
\textsc{DP-FedAvg+Ditto} & - &$75.39$ &$76.1 (1.73)$ & - & - &$71.3 (1.47)$ &- &- \\ \hline
\textsc{HDP-FedAvg+Ditto} & - &$74.96$ &$75.6(1.69)$& $77.84(1.47)$& $2.24$& $71.44(1.5)$& $70.03(1.18)$& $-1.4$ \\ \hline
\textsc{FedHDP} & $r\!=\!0$ & $72.77$ &$75.34(2.6)$ &$85.7(1.14)$&$10.36$ &$75.34(2.6)$&$72.2(1.22)$ &$-3.14$ \\ \hline
\textsc{FedHDP} & $r\!=\!0.001$ &$73.66$&$76.22(2.44)$& $86.04(1.2)$& $9.82$& $72.78(1.51)$& $71.74(1.34)$& $-1.04$ \\ \hline
\textsc{FedHDP} & $r\!=\!0.01$ &$74.75$&$77.16(2.26)$& $86.4(1.07)$& $9.24$& $73.24(1.52)$& $71.49(1.29)$& $-1.76$ \\ \hline
\textsc{FedHDP} & $r\!=\!0.025$ &$75.37$ &$77.66(2.06)$& $86.56(1)$& $8.9$& $73.11(1.55)$& $71.62(1.18)$& $-1.49$ \\ \hline
\textsc{FedHDP} & $r\!=\!0.1$ &$76.47$ &$77.99(1.68)$& $84.36(1.31)$& $6.37$& $72.3(1.5)$& $69.93(1.15)$& $-2.37$ \\ \hline
\textsc{FedHDP} & $r\!=\!0.5$ &$76.11$ &$76.69(1.62)$& $80.82(1.3)$& $4.13$& $71.32(1.56)$& $70.11(1.26)$& $-1.2$ \\ \hline

\rowcolor{Gray} \multicolumn{9}{|c|}{$\lambda_{\text{p}} = \lambda_{\text{np}} = 0.05$} \\ \hline
\textsc{Non-Private+Ditto} & - &$81.95$ & - &$82.09(1.38)$&- &-&$82.89(1.13)$ &- \\ \hline
\textsc{DP-FedAvg+Ditto} & - &$75.42$ &$75.86(1.82)$ & - & - &$74.69(1.29)$ &- &- \\ \hline
\textsc{HDP-FedAvg+Ditto} & - &$75.12$ &$75.87(1.65)$& $78.59(1.58)$& $2.72$& $74.67(1.34)$& $75.95(1.12)$& $1.28$ \\ \hline
\textsc{FedHDP} & $r\!=\!0$ & $72.65$ &$75.9(2.5)$& $86.19(1.27)$& $10.29$& $80.59(1.13)$& $81.97(0.88)$& $1.38$ \\ \hline
\textsc{FedHDP} & $r\!=\!0.001$ &$73.31$&$75.9(2.5)$& $86.19(1.27)$& $10.29$& $80.59(1.13)$& $81.97(0.88)$& $1.38$ \\ \hline
\textsc{FedHDP} & $r\!=\!0.01$ &$74.68$&$77.16(2.27)$& $86.25(1.05)$& $9.09$& $80.74(1.06)$& $82.13(0.98)$& $1.38$ \\ \hline
\textsc{FedHDP} & $r\!=\!0.025$ &$75.22$ &$77.43(2.09)$& $85.95(1.12)$& $8.52$& $80(1.16)$& $80.99(0.92)$& $1.01$ \\ \hline
\rowcolor{LightCyan} \textsc{FedHDP} & $r\!=\!0.1$ &$76.52$ &$77.91(1.67)$& $83.9(1.27)$& $5.99$& $77.9(1.22)$& $79.15(0.99)$& $1.25$ \\ \hline
\textsc{FedHDP} & $r\!=\!0.5$ &$76.15$ &$76.55(1.68)$& $80.04(1.62)$& $3.49$& $75.43(1.25)$& $77.13(1.17)$& $1.7$ \\ \hline

\rowcolor{Gray} \multicolumn{9}{|c|}{$\lambda_{\text{p}} = \lambda_{\text{np}} = 0.25$} \\ \hline
\textsc{Non-Private+Ditto} & - &$81.66$ & - &$81.79(1.38)$&- &-&$84.46(0.89)$ &- \\ \hline
\textsc{DP-FedAvg+Ditto} & - &$75.99$ &$76.56(1.6)$ & - & - &$73.06(1.46)$ &- &- \\ \hline
\textsc{HDP-FedAvg+Ditto} & - &$75.31$ &$75.67(1.71)$& $78.88(1.59)$& $3.21$& $72.58(1.45)$& $74.98(1.43)$& $2.4$ \\ \hline
\textsc{FedHDP} & $r\!=\!0$ & $72.89$ &$75.5(2.56)$& $86.09(1.28)$& $10.6$& $75.5(2.56)$& $84.77(0.8)$& $9.28$ \\ \hline
\textsc{FedHDP} & $r\!=\!0.001$ &$73.41$&$76.01(2.51)$& $85.99(1.13)$& $9.97$& $80.98(1.06)$& $84.71(0.83)$& $3.73$ \\ \hline
\rowcolor{LightCyan} \textsc{FedHDP} & $r\!=\!0.01$ &$74.86$&$77.31(2.18)$& $86.73(0.98)$& $9.42$& $81.19(1.02)$& $84.68(0.78)$& $3.49$ \\ \hline
\textsc{FedHDP} & $r\!=\!0.025$ &$75.41$ &$77.68(2.1)$& $86.23(1.03)$& $8.55$& $80.01(1.1)$& $83.2(0.8)$& $3.19$ \\ \hline
\textsc{FedHDP} & $r\!=\!0.1$ &$76.62$ &$77.82(1.68)$& $83.35(1.27)$& $5.52$& $76.99(1.24)$& $78.96(1.04)$& $1.97$ \\ \hline
\textsc{FedHDP} & $r\!=\!0.5$ &$75.89$ &$76.71(1.65)$& $80.01(1.4)$& $3.3$& $73.48(1.37)$& $75.49(1.57)$& $2.01$ \\ \hline

\end{tabular}
}
\end{center}
\end{table}

\clearpage

\section{Broader Impact \& Limitations}\label{broader_impact}
In this paper, we investigated a heterogeneous privacy setup where different clients may have different levels of privacy protection guarantees, and in particular explored an extreme setup where some clients may opt out of privacy guarantees, to gain improvements in performance. However, the choice to loosen privacy requirements is heavily dependent on the client, the setting, and their valuation of their data. Moreover, since the algorithm orients the model towards the less private clients, it may introduce unfairness for the more private clients. Additionally, the server may have its own requirements during training, for example, a lower limit to the fraction of less private clients or vice versa, where the privacy choices may be overridden by the server. Overall, we believe that the interplay between all of these different societal aspects need to be carefully studied before the proposed mechanisms in this paper can be practically used.

We acknowledge that some of the assumptions in the theoretical study of the federated linear regression and federated point estimation setups are unrealistic, but similar assumptions have been made in prior work of \cite{li2021ditto}. For example, we neglected the effect of clipping, assumed that all clients have the same number of samples, and assumed the covariance matrix is diagonal. On the other hand, for more complex models such as the ones used in the experiments, finding the best values of weights to be used in the aggregator at the server as well as the personalization parameters for each client is not straightforward, as some of the theoretical constructs in this paper are not estimable from data. Nevertheless, we believe that the theoretical studies in this paper can be used to build intuition about heterogeneous privacy setups and can be used as guiding principles for designing new algorithms.

Finally, the \textsc{FedHDP} algorithm comes with two additional hyperparameters compared with \textsc{DP-FedAvg}: $r$ (the weight ratio of private and non-private clients at the server), and $\lambda_j$ (the degree of personalization at client $c_j$). In this paper, we chose $r$ based on grid search, however that will naturally incur a loss of privacy that we did not carefully study. In fact, we assumed there is no privacy loss due to the tuning of such hyperparameter in our paper. Having said that, recent work by~\cite{papernot2021hyperparameter} suggests that the privacy loss due to such hyperparameter tuning based on private training runs might be manageable but the exact interplay remains to be studied in future work.

\end{document}